\documentclass{new_tlp}
\usepackage{times}
\usepackage{helvet}
\usepackage{courier}
\usepackage{amsmath}
\usepackage{algorithm}

\usepackage{amssymb}
\usepackage{graphicx}
\usepackage{color}
\usepackage{url}
\usepackage{listings}

\newcommand{\cblu}{\color{blue}}
\newcommand{\cred}{\color{black}}

\renewcommand{\vec}[1]{\mathbf{#1}}

\long\def\BOC#1\EOC{\message{(Commented text )}}
\long\def\BOCC#1\EOCC{\message{(Commented text )}}
\long\def\BOCCC#1\EOCCC{\message{(Commented text )}}
\long\def\optional#1{\empty}
\long\def\NB#1{\bigskip[{\cblu {\bf N.B.} #1}]\bigskip}
\long\def\NB#1{}
\long\def\NBB#1{}

\def\o{\overline}
\def\ar{\leftarrow}
\def\bi{\begin{itemize}}
\def\ei{\end{itemize}}
\def\beq{\begin{equation}}
\def\eeq#1{\label{#1}\end{equation}}
\def\ba{\begin{array}}
\def\ea{\end{array}}

\def\sm{\rm SM}

\def\no{\i{not}}
\def\sneg{\sim\!\!}

\def\ar{\leftarrow}
\def\rar{\rightarrow}

\def\no{\i{not}}
\def\mvis{\!=\!}
\def\false{\hbox{\sc false}}
\def\true{\hbox{\sc true}}

\def\i#1{\hbox{\itshape #1\/}}

\def\proof{\noindent{\bf Proof}.\hspace{3mm}}
\def\qed{\quad \vrule height7.5pt width4.17pt depth0pt \medskip}
     \def\smmodels{\models_{\text{\sm}}}

\DeclareSymbolFont{AMSa}{U}{msa}{m}{n}
\DeclareMathSymbol{\square}{\mathord}{AMSa}{"03}

\def\mvis{\!=\!}
\def\mu#1{\mathit{\underline{#1}}}
\mathchardef\mhyphen="2D 
\def\dtlpmln{{\rm DT}{\mhyphen}{\rm LP}^{\rm{MLN}}}
\def\lpmln{{\rm LP}^{\rm{MLN}}}
\def\pbcp{p{\cal BC}+}

\def\caused{\hbox{\bf caused}}
\def\iif{\hbox{\bf if}}
\def\init{\hbox{\bf initially}}
\def\after{\hbox{\bf after}}

\def\causes{\hbox{\bf causes}}
\def\inertial{\hbox{\bf inertial}}
\def\default{\hbox{\bf default}}
\def\reward{\hbox{\bf reward}}

\def\constraint{\hbox{\bf constraint}}

\def\:{\!:\!}

\def\bi{\begin{itemize}}
\def\ei{\end{itemize}}
\def\ii{\medskip\item}

\newtheorem{prop}{Proposition}
\newtheorem{thm}{Theorem}
\newtheorem{cor}{Corollary}
\newtheorem{definition}{Definition}
\newtheorem{lemma}{Lemma} 
\newtheorem{example}{Example}

\begin{document}


\title[Elaboration Tolerant Representation of MDP via $\pbcp$]{Elaboration Tolerant Representation of Markov Decision Process via Decision-Theoretic Extension of Probabilistic Action Language $\pbcp$}


\author[Wang \& Lee]{Yi Wang \and Joohyung Lee \\
School of Computing, Informatics, and Decision Systems Engineering \\
Arizona State University, Tempe, USA \\
\email{\{ywang485, joolee\}@asu.edu}
}

\maketitle

\begin{abstract}
We extend probabilistic action language $\pbcp$ with the notion of utility in decision theory. The semantics of the extended $\pbcp$ can be defined as  a shorthand notation for a decision-theoretic extension of the probabilistic answer set programming language $\lpmln$. Alternatively, the semantics of $\pbcp$ can also be defined in terms of Markov Decision Process (MDP), which in turn allows for representing MDP in a succinct and elaboration tolerant way as well as leveraging an MDP solver to compute a $\pbcp$ action description. The idea led to the design of the system {\sc pbcplus2mdp}, which can find an optimal policy of a $\pbcp$ action description using an MDP solver. 
This paper is under consideration in Theory and Practice of Logic Programming (TPLP).

\keywords{Answer Set Programming \and Action Language \and Markov Decision Process}
\end{abstract}


\section{Introduction}


Many problems in Artificial Intelligence are about making decisions on actions to take. The chosen actions should maximize the agent's utility, which is a quantitative measurement of the value or desirability to the agent. Since actions may also have stochastic effects, the main computational task is, rather than to find a sequence of actions that leads to a goal, to find an optimal policy, that states which actions to execute in each state to achieve the maximum expected utility. 


While a few decades of research on action languages has produced several expressive  languages, such as ${\cal A}$ \cite{gel93a}, ${\cal B}$ \cite{gel98}, ${\cal C}$+ \cite{giu04}, ${\cal BC}$ \cite{lee13action}, and ${\cal BC}$+ \cite{babb15action1}, that are able to describe actions and their effects in a succinct and elaboration tolerant way, these languages are not equipped with constructs to represent stochastic actions and the utility of a decision.
In this paper, we present an action language that overcomes the limitation. Our method is to equip probabilistic action language $\pbcp$ \cite{lee18aprobabilistic} with the notion of utility and define policy optimization problems in that language.

Following the way that $\pbcp$ is defined as a shorthand notation of probabilistic answer set programming language $\lpmln$ for describing a probabilistic transition system, we first extend $\lpmln$
by associating a utility measure to each soft stable model in addition to its already defined probability. We call this extension $\dtlpmln$.
Next, we define a decision-theoretic extension of $\pbcp$ as a shorthand notation for $\dtlpmln$.
It turns out that the semantics of $\pbcp$ can also be directly defined in terms of Markov Decision Process (MDP) \cite{bellman57markovian}, which in turn allows us to define MDP in a succinct and elaboration tolerant way.
The result is theoretically interesting as it formally relates action languages to MDP despite their different origins, and furthermore justifies the semantics of the extended $\pbcp$ in terms of MDP. It is also computationally interesting because it allows for applying a number of algorithms developed for MDP to  computing $\pbcp$. Based on this idea, we design the system {\sc pbcplus2mdp} \cite{yi_wang_2020_3726430}\footnote{\url{https://github.com/ywang485/pbcplus2mdp}}, which turns a $\pbcp$ action description into the input language of an MDP solver, and leverages MDP solving to find an optimal policy for the $\pbcp$ action description.

The extended $\pbcp$ can thus be viewed as a high-level representation of MDP that allows for compact and elaboration tolerant encodings of sequential decision problems. Compared to other MDP-based planning description languages, such as PPDDL~\cite{younes04ppddl1} and RDDL~\cite{sanner10relational}, it inherits the nonmonotonicity of the stable model semantics to be able to compactly represent recursive definitions and indirect effects of actions, which can save the state space significantly. The following action domain is such an example.
\begin{example}\label{eg:block}
{\sl 
{\bf Robot and Blocks}\hspace{0.2cm} There are two rooms ${\tt R1}$, ${\tt R2}$, and three blocks ${\tt B1}$, ${\tt B2}$, ${\tt B3}$ that are originally located in ${\tt R1}$. A robot can stack one block on top of another block if the two blocks are in the same room. The robot can also move a block to a different room, resulting in all blocks above it also moving if successful (with probability $p$). Each moving action has a cost of $1$. What is the best way to move all blocks to ${\tt R2}$?
}
\end{example}
In this example, the effect of moving a block is stochastic. The best way of moving is defined as a moving policy that minimizes the expected total cost. Representing the cost of moving requires a notion of utility. Successfully moving a block has an indirect and recursive effect that the block on top of it is also moved. We show how this example can be represented in $\pbcp$ in Section \ref{sec:block-world}, and how the query can be answered with system {\sc pbcplus2mdp} in Section \ref{sec:system-pbcplus2mdp}.

We summarize our contribution as follows:
\begin{itemize}
\item We extended $\lpmln$ with the notion of utility, resulting in $\dtlpmln$; we developed an approximate algorithm for maximizing expected utility in $\dtlpmln$;
\item Based on $\dtlpmln$, we extended $\pbcp$ with the notion of utility;
\item We showed that the semantics of $\pbcp$ can be alternatively defined in terms of Markov Decision Process;
\item We demonstrated how $\pbcp$ can serve as an elaboration tolerant representation of MDP;
\item We developed a prototype system {\sc pbcplus2mdp}, for finding optimal policies of $\pbcp$ action descriptions using an MDP solver.  
\end{itemize}

\BOCC
Based on our conference paper \cite{wang2019elaboration}, we have made the following improvements in this journal version:

\begin{itemize}
\item We developed an algorithm for maximizing expected utility in $\dtlpmln$ and included some experimental result in section \ref{sec:dt-lpmln};
\item We have extended the preliminary section for a more self-contained presentation;
\item We included proofs for our theoretical results;
\item More implementation details about system {\sc pbcplus2mdp} can be found in \ref{sec:pbcplus2mdp-appendix}.
\end{itemize}
\EOCC

This paper is an extended version of \cite{wang19elaboration}, with the following advancements:

\begin{itemize}
\item We developed an algorithm for maximizing expected utility in $\dtlpmln$ and included some experimental result;
\item We have extended the preliminary section for a more self-contained presentation;
\item We included proofs for our theoretical results;
\item More implementation details about system {\sc pbcplus2mdp} can be found in the appendix.
\end{itemize}

This paper is organized as follows. After Section \ref{sec:preliminaries} reviews preliminaries,  Section~\ref{sec:dt-lpmln} extends $\lpmln$ with the notion of utility, through which we define the extension of $\pbcp$ with utility in Section \ref{sec:pbcplus-rewards}. 
Section~\ref{sec:block-world} defines $\pbcp$ as a high-level representation language for MDP, and Section \ref{sec:system-pbcplus2mdp} presents the prototype system {\sc pbcplus2mdp}. We discuss the related work in Section \ref{sec:related-work}.

\section{Preliminaries}\label{sec:preliminaries}

\BOCC
Due to the space limit, the reviews are brief. We refer the reader to the original papers \cite{lee16weighted,lee18aprobabilistic}, or the technical report of this paper \cite{wang19elaboration-arxiv} for the reviews of preliminaries. The technical report also contains all proofs and experiments with the system {\sc pbcplus2mdp}. 
\EOCC

\BOCC
\subsection{Review: Language $\lpmln$}

 An $\lpmln$ program is a finite set of weighted rules $w: R$ where $R$ is a rule and $w$ is a real number (in which case, the weighted rule is called {\em soft}) or $\alpha$ for denoting the infinite weight (in which case, the weighted rule is called {\em hard}). Throughout the paper, we assume that the language is propositional. Schematic variables can be introduced via grounding as standard in answer set programming. 
 
For any $\lpmln$ program $\Pi$ and any interpretation~$I$, 
$\overline{\Pi}$ denotes the usual (unweighted) logic program obtained from $\Pi$ by dropping the weights, and
${\Pi}_I$ denotes the set of $w: R$ in $\Pi$ such that $I\models R$.

%
Given a ground $\lpmln$ program $\Pi$, $\sm[\Pi]$ denotes the set of ``soft" stable models:
\[
\ba l
\{I\mid \text{$I$ is a (deterministic) stable model of $\Pi_I$ that satisfies all hard rules in $\Pi$} \} .
\ea
\]
The weight $W_\Pi(I)$ of an interpretation $I$ is defined as 
$exp\Bigg(\sum\limits_{w:R\;\in\; {\Pi}_I} w\Bigg)$ if $I\in\sm[\Pi]$, and 0 otherwise.
\BOC
\[
 W_\Pi(I) =
\begin{cases}
  exp\Bigg(\sum\limits_{w:R\;\in\; {\Pi}_I} w\Bigg) & 
      \text{if $I\in\sm[\Pi]$}; \\
  0 & \text{otherwise},
\end{cases}
\]
\EOC
The probability $P_\Pi(I)$ of $I$ is defined as
$
  \frac{W_\Pi(I)}{\sum\limits_{J\in {\rm SM}[\Pi]}{W_\Pi(J)}}
  $.

\EOCC

\subsection{Review: The Stable Model Semantics}

We first review the definition of a (deterministic) stable model. Given a propositional signature $\sigma$, we consider {\em rules} over $\sigma$ of the form
\begin{equation}
\label{eq:rule}
\ba l
   A_1;\dots; A_k\ \ar\ 
      A_{k+1}, \dots, A_m, \no\ A_{m+1}, \dots, \no\ A_n, \no\ \no\ A_{n+1},\dots, \no\ \no\ A_p
\ea 
\end{equation}
($0\le k\le m\le n\le p$) 
where all $A_i$ are atoms of $\sigma$. $ A_1;\dots; A_k$ is called the {\em head} of the rule and 
\[
A_{k+1}, \dots, A_m, \no\ A_{m+1}, \dots, \no\ A_n, \no\ \no\ A_{n+1},\dots, \no\ \no\ A_p
\]
is called the {\em body} of the rule. We write $\{A_1\}^{\rm ch}\ar \i{Body}$ to denote the rule $A_1\ar \i{Body}, \no\ \no\ A_1$. This expression is called a ``choice rule'' in Answer Set Programming. 

We will often identify \eqref{eq:rule} with the implication:
\beq
\ba l
   A_1\lor\dots\lor A_k\ \ar\ 
      A_{k+1}\!\land\!\dots\!\land\! A_m\!\land\! \neg A_{m+1}\!\land\!\dots\!\land\!\neg A_n 
       \!\land\!\neg\neg A_{n+1}\!\land\!\dots\!\land\!\neg\neg A_p\ .
\ea
\eeq{rule-impl}

A {\sl logic program} is a finite set of rules. A logic program is
called {\sl ground} if it contains no variables. For an interpretation $I$ and a formula $F$, we use $I\models F$ to denote ``$I$ satisfies $F$''. 
We say that an Herbrand interpretation $I$ is a {\em model} of a ground
program $\Pi$ if $I$ satisfies all implications~\eqref{rule-impl}
in~$\Pi$. 

Such models can be divided into two groups: ``stable''
and ``non-stable'' models, which are distinguished as follows. The {\em reduct} of $\Pi$ relative to $I$, denoted
$\Pi^I$, consists of 
``$
  A_1\lor\dots\lor A_k\ \ar\ A_{k+1}\land\dots\land A_m
$''
for all rules~\eqref{rule-impl} in $\Pi$ such
that $I\models  \neg A_{m+1}\land\dots\land\neg A_n\land\neg\neg
A_{n+1}\land\dots\land\neg\neg A_p$.

\begin{definition}
The Herbrand interpretation $I$ is called a {\em (deterministic) stable model} of~$\Pi$ (denoted by $I\smmodels\Pi$) if $I$ is a minimal Herbrand model of $\Pi^I$. 
(Minimality is in terms of set inclusion. We identify an Herbrand
interpretation with the set of atoms that are true in it.)
\end{definition}

%
For example, the stable models of the program
\beq
\ba {lllllll}
 P \ar Q  & \hspace{5mm} & 
 Q \ar P  & \hspace{5mm} &  
 P\ar \no\ R & \hspace{5mm} & 
 R \ar \no\ P 
\ea
\eeq{ex1}
are $\{P, Q\}$ and $\{R\}$. The reduct relative to
$\{P,Q\}$ is $\{P\ar Q.\ \ \ Q\ar P.\ \ \ P.\}$, for which $\{P,Q\}$ is the
minimal model; the reduct relative to 
$\{R\}$ is \hbox{$\{P\ar Q.\ \ \ Q\ar P.\ \ \ R.\}$}, for which $\{R\}$ is the minimal model.

The definition is extended to any non-ground program~$\Pi$ by identifying it with $gr_\sigma[\Pi]$, the {\em ground program} obtained from~$\Pi$ by replacing every variable with every ground term of~$\sigma$.

The semantics is extended to allow some useful constructs, such as aggregates and abstract constraints (e.g., \cite{nie00,fab04,fer05,son06,pel07}), which are proved to be useful in many KR domains. 

\subsection{Review: Language $\lpmln$}

We review the definition of $\lpmln$ from~\cite{lee16weighted}.
\begin{definition}
An $\lpmln$ program is a finite set of weighted rules $w: R$ where $R$ is a rule of the form \eqref{eq:rule} and $w$ is a real number (in which case, the weighted rule is called {\em soft}) or $\alpha$ for denoting the infinite weight (in which case, the weighted rule is called {\em hard}). 
\end{definition}
 
For any $\lpmln$ program $\Pi$ and any interpretation~$I$, 
$\overline{\Pi}$ denotes the usual (unweighted) logic program obtained from $\Pi$ by dropping the weights, and
${\Pi}_I$ denotes the set of $w: R$ in $\Pi$ such that $I\models R$.

In general, an $\lpmln$ program may even have stable models that violate some hard rules, which encode definite knowledge. However, throughout the paper, we restrict attention to $\lpmln$ programs whose stable models do not violate hard rules. 
\begin{definition}
Given a ground $\lpmln$ program $\Pi$, $\sm[\Pi]$ denotes the set
\[
\ba l
\{I\mid \text{$I$ is a (deterministic) stable model of $\Pi_I$ that satisfies all hard rules in $\Pi$} \} .
\ea
\]
The weight of an interpretation $I$, denoted $W_{\Pi}(I)$, is defined as\footnote{$exp$ stands for the natural exponential function.} 
\[
 W_\Pi(I) =
\begin{cases}
  exp\Bigg(\sum\limits_{w:R\;\in\; {\Pi}_I} w\Bigg) & 
      \text{if $I\in\sm[\Pi]$}; \\
  0 & \text{otherwise},
\end{cases}
\]
and the probability of $I$, denoted $P_\Pi(I)$, is defined as
\[
  P_\Pi(I)  = 
  \frac{W_\Pi(I)}{\sum\limits_{J\in {\rm SM}[\Pi]}{W_\Pi(J)}}. 
\]
\end{definition}


\subsection{Review: Multi-Valued Probabilistic Programs} \label{ssec:mvpp}
Multi-valued probabilistic programs \cite{lee16weighted} are a simple fragment of $\lpmln$ that allows us to represent probability more naturally. 

We assume that the propositional signature $\sigma$ is constructed from ``constants'' and their ``values.'' A {\em constant} $c$ is a symbol that is associated with a finite set $\i{Dom}(c)$, called the {\em domain}. 
The signature $\sigma$ is constructed from a finite set of constants, consisting of atoms $c\!=\!v$~\footnote{%
Note that here ``='' is just a part of the symbol for propositional atoms, and is not  equality in first-order logic. }
for every constant $c$ and every element $v$ in $\i{Dom}(c)$.

If the domain of~$c$ is $\{\false,\true\}$ then we say that~$c$ is {\em Boolean}, and abbreviate $c\mvis\true$ as $c$ and $c\mvis\false$ as~$\sneg c$. 

We assume that constants are divided into {\em probabilistic} constants and {\em non-probabilistic} constants.
A multi-valued probabilistic program ${\bf \Pi}$ is a tuple $\langle \i{PF}, \Pi \rangle$, where
\begin{itemize}
\item $\i{PF}$ contains \emph{probabilistic constant declarations} of the following form:
\begin{equation}\label{eq:probabilistic-constant-declaration}
p_1::\ c\mvis v_1\mid\dots\mid p_n::\ c\mvis v_n
\end{equation}
one for each probabilistic constant $c$, where $\{v_1,\dots, v_n\}=\i{Dom}(c)$, $v_i\ne v_j$, $0\leq p_1,\dots,p_n\leq1$ and $\sum_{i=1}^{n}p_i=1$. We use $M_{\bf \Pi}(c=v_i)$ to denote $p_i$.
In other words, $\i{PF}$ describes the probability distribution over each ``random variable''~$c$. 

\item $\Pi$ is a set of rules such that the head contains no probabilistic constants.
\end{itemize}

Such a program is called a {\em multi-valued probabilistic program}. The semantics of such a program ${\bf \Pi}$ is defined as a shorthand for $\lpmln$ program $T({\bf \Pi})$ of the same signature as follows.
\begin{itemize}
\item For each probabilistic constant declaration (\ref{eq:probabilistic-constant-declaration}), $T({\bf \Pi})$ contains, 
for each $i=1,\dots, n$,
(i) $ln(p_i):  c\mvis v_i$  if $0<p_i<1$; 
(ii) $\alpha:\ c\mvis v_i$ if $p_i=1$;
(iii) $\alpha:\ \bot\ar c\mvis v_i$ if $p_i=0$.

\item  For each rule $\i{Head}\ar\i{Body}$ in $\Pi$, $T({\bf \Pi})$ contains
$
\alpha:\ \  \i{Head}\ar\i{Body}. 
$

\ii For each constant $c$, $T({\bf \Pi})$ contains the uniqueness of value constraints
\beq
\ba {rl}
   \alpha: & \bot \ar c\mvis v_1\land c=v_2 
\ea 
\eeq{uc}
for all $v_1,v_2 \in\i{Dom}(c)$ such that $v_1\ne v_2$, and the existence of value constraint
\beq
\ba {rl}
  \alpha: & \bot \ar \neg \bigvee\limits_{v \in {Dom}(c)} c\mvis v\ .
\ea 
\eeq{ec}
\end{itemize}

In the presence of the constraints \eqref{uc} and \eqref{ec}, assuming $T({\bf \Pi})$ has at least one (probabilistic) stable model that satisfies all the hard rules, a (probabilistic) stable model $I$ satisfies $c=v$ for exactly one value $v$, so we may identify $I$ with the value assignment that assigns $v$ to $c$.

\subsection{Review: Action Language $\pbcp$}\label{ssec:pbcplus}
In this section, we review the syntax and semantics of $\pbcp$ from \cite{lee18aprobabilistic}.

\subsubsection{Syntax of $p\cal{BC}+$}
We assume a propositional signature~$\sigma$ as defined in Section~\ref{ssec:mvpp}.
We further assume that the signature of an action description is divided into four groups: {\em fluent constants}, {\em action constants},  {\em pf (probability fact) constants} and {\em  initpf (initial probability fact) constants}. Fluent constants are further divided into {\em regular} and {\em statically determined}. The domain of every action constant is Boolean. 
A {\em fluent formula} is a formula such that all constants occurring in it are fluent constants. 

The following definition of $p\cal{BC}$+ is based on the definition of ${\cal BC}$+ language from \cite{babb15action1}.

A {\em static law} is an expression of the form
\begin{equation}\label{eq:static-law}
\caused\ F\ \iif\ G
\end{equation}
where $F$ and $G$ are fluent formulas.


A {\em fluent dynamic law} is an expression of the form
\begin{equation}\label{eq:fluent-dynamic-law}
\caused\ F\ \iif\ G\ \after\ H
\end{equation}
where $F$ and $G$ are fluent formulas and $H$ is a formula, provided that $F$ does not contain statically determined constants and $H$ does not contain initpf constants.

A {\em pf constant declaration} is an expression of the form
\begin{equation}\label{eq:pf-declare-no-time}
   \caused\ \i{c}=\{v_1:p_1, \dots, v_n:p_n\}
\end{equation}
where $\i{c}$ is a pf constant with domain $\{v_1, \dots, v_n\}$, $0<p_i<1$ for each $i\in\{1, \dots, n\}$\footnote{We require $0<p_i<1$ for each $i\in\{1, \dots, n\}$ for the sake of simplicity. On the other hand, if $p_i=0$ or $p_i=1$ for some $i$, that means either $v_i$ can be removed from the domain of $c$ or there is not really a need to introduce $c$ as a pf constant. So this assumption does not really sacrifice expressivity.}, and $p_1+\cdots+p_n=1$. In other words, \eqref{eq:pf-declare-no-time} describes the probability distribution of $c$.

An {\em initpf constant declaration} is an expression of the form (\ref{eq:pf-declare-no-time}) where $c$ is an initpf constant. 

An {\em initial static law} is an expression of the form
\begin{equation}\label{eq:init-static-law}
\init\ F\ \iif\ G
\end{equation}
where $F$ is a fluent constant and $G$ is a formula that contains neither action constants nor pf constants. 

A {\em causal law} is a static law, a fluent dynamic law, a pf constant  declaration, an initpf constant declaration, or an initial static law. An {\em action description} is a finite set of causal laws.

We use $\sigma^{fl}$ to denote the set of fluent constants, $\sigma^{act}$ to denote the set of action constants, $\sigma^{pf}$ to denote the set of pf constants, and $\sigma^{initpf}$ to denote the set of initpf constants. For any signature $\sigma^\prime$ and any $i\in\{0, \dots, m\}$, we use $i:\sigma^\prime$ to denote the set
$\{i:a \mid a\in\sigma^\prime\}$.

By $i:F$ we denote the result of inserting $i:$ in front of every occurrence of every constant in formula $F$. This notation is straightforwardly extended when $F$ is a set of formulas.

\subsubsection{Semantics of $p\cal{BC}+$}
Given an integer $m\geq 0$ denoting the maximum length of histories, the semantics of an action description $D$ in $p{\cal BC}$+ is defined by a reduction to multi-valued probabilistic program $Tr(D, m)$, which is the union of two subprograms $D_m$ and $D_{init}$ as defined below. 

For an action description $D$ of a signature $\sigma$, we define a sequence of multi-valued probabilistic program $D_0, D_1, \dots,$ so that the stable models of $D_m$ can be identified with the paths in the ``transition system'' described by $D$.

The signature $\sigma_m$ of $D_m$ consists of atoms of the form $i:c=v$ such that
\begin{itemize}
\item for each fluent constant $c$ of $D$, $i\in\{0, \dots, m\}$ and $v\in Dom(c)$,
\item for each action constant or pf constant $c$ of $D$, $i\in\{0, \dots, m-1\}$ and $v\in Dom(c)$.
\end{itemize}

For $x\in \{act, fl, pf\}$, we use $\sigma^{x}_m$ to denote the subset of $\sigma_m$
\[
\{i:c=v \mid \text{$i:c=v\in \sigma_m$ and $c\in\sigma^{x}$}\}.
\]
For $i\in\{0, \dots, m\}$, we use $i:\sigma^{x}$ to denote the subset of $\sigma_m^{x}$
\[
\{i:c=v\mid i:c=v\in \sigma_m^{x}\}.
\]

We define $D_m$ to be the multi-valued probabilistic program  $\langle PF, \Pi\rangle$, where $\Pi$ is the conjunction of

\begin{equation}\label{eq:static-law-asp}
i:F \leftarrow i:G
\end{equation}
for every static law (\ref{eq:static-law}) in $D$ and every $i\in\{0, \dots, m\}$,

\begin{equation}\label{eq:fluent-dynamic-law-asp}
i\!+\!1:F\leftarrow (i\!+\!1:G)\wedge(i:H)
\end{equation}
for every fluent dynamic law (\ref{eq:fluent-dynamic-law}) in $D$ and every $i\in\{0, \dots, m-1\}$,

\begin{equation}\label{eq:init-fluent-choice}
\{0\!:\!c=v\}^{\rm ch}
\end{equation}
for every regular fluent constant $c$ and every $v\in Dom(c)$,
 
\begin{equation}\label{eq:action-choice}
\{i:c=\true\}^{\rm ch}, \ \ \ \ 
\{i:c=\false\}^{\rm ch}
\end{equation}
for every action constant $c$,

and $PF$ consists of 
\begin{equation}\label{eq:prod-declaration-mvpp}
   p_1::\ i:pf=v_1 \mid \dots \mid p_n::\ i:pf=v_{n}
\end{equation}
($i=0,\dots,m-1$) for each pf constant declaration \eqref{eq:pf-declare-no-time} in $D$ that describes the probability distribution of $\i{pf}$.

In addition, we define the program $D_{init}$, whose signature is $0\!:\!\sigma^{initpf}\cup 0\!:\!\sigma^{fl}$.
$D_{init}$ is the multi-valued probabilistic program
\[
D_{init} = \langle PF^{init}, \Pi^{init}\rangle
\]
where $\Pi^{init}$ consists of the rule
\[
\bot\leftarrow \neg(0\!:\!F)\land 0\!:\!G
\]
for each initial static law (\ref{eq:init-static-law}),
and $PF^{init}$ consists of 
\[
p_1::\ 0\!:\!pf=v_1\ \ \mid\ \  \dots\ \  \mid\ \  p_n::\ 0\!:\!pf=v_n
\]
for each initpf constant declaration (\ref{eq:pf-declare-no-time}).

We define $Tr(D, m)$ to be the union of the two multi-valued probabilistic program \\
$
\langle  PF\cup PF^{init}, \Pi\cup\Pi^{init} \rangle.
$

For any $\lpmln$ program $\Pi$ of signature $\sigma$ and a value assignment $I$ to a subset $\sigma'$ of $\sigma$, we say $I$ is a {\em residual (probabilistic) stable model} of $\Pi$ if there exists a value assignment $J$ to $\sigma\setminus \sigma'$ such that $I\cup J$ is a (probabilistic) stable model of $\Pi$.

For any value assignment $I$ to constants in $\sigma$, by $i\!:\!I$ we denote the value assignment to constants in $i\!:\!\sigma$ so that $i\!:\!I\models (i\!:\!c)=v$ iff $I\models c=v$.

We define a {\em state} as an interpretation $I^{fl}$ of $\sigma^{fl}$ such that $0\!:\!I^{fl}$ is a residual (probabilistic) stable model of $D_0$. A {\em transition} of $D$ is a triple $\langle s, e, s^\prime\rangle$  where $s$ and $s^\prime$ are interpretations of $\sigma^{fl}$ and $e$ is a an interpretation of $\sigma^{act}$ such that $0\!:\!s \cup 0\!:\!e \cup 1:s^\prime$ is a residual stable model of $D_1$. A {\em pf-transition} of $D$ is a pair $(\langle s, e, s^\prime\rangle, pf)$, where $pf$ is a value assignment to $\sigma^{pf}$ such that $0\!:\!s\cup 0\!:\!e \cup 1:s^\prime \cup 0\!:\!pf$ is a stable model of $D_1$. 

\begin{definition}
A {\em (probabilistic) transition system} $T(D)$ represented by a probabilistic action description $D$ is a labeled directed graph such that the vertices are the states of $D$, and the edges are obtained from the transitions of $D$: for every transition $\langle s, e, s^\prime\rangle$  of $D$, an edge labeled $e: p$ goes from $s$ to $s^\prime$, where $p=Pr_{D_m}(1\!:\!s^\prime \mid 0\!:\!s, 0\!:\!e)$. The number $p$ is called the {\em transition probability} of $\langle s, e ,s^\prime\rangle$.
\end{definition}

The soundness of the definition of a probabilistic transition system relies on the following proposition. 
\begin{prop}\label{prop:state-in-transition}
For any transition $\langle s, e, s^\prime \rangle$, $s$ and $s^\prime$ are states.
\end{prop}

\cite{lee18aprobabilistic} make the following simplifying assumptions on action descriptions:

\begin{enumerate}
\item {\bf No Concurrency}: For all transitions $\langle s, e, s'\rangle$, we have $e(a)=t$ for at most one $a\in \sigma^{act}$;
\item {\bf Nondeterministic Transitions are Controlled by pf constants}: For any state $s$, any value assignment $e$ of $\sigma^{act}$ such that at most one action is true, and any value assignment $pf$ of $\sigma^{pf}$, there exists exactly one state $s^\prime$ such that $(\langle s, e, s^\prime\rangle, pf)$ is a pf-transition;
\item {\bf Nondeterminism on Initial States are Controlled by Initpf constants}: Given any assignment $pf_{init}$ of $\sigma^{initpf}$, there exists exactly one assignment $fl$ of $\sigma^{fl}$ such that $0\!:\!pf_{init}\cup 0\!:\!fl$ is a stable model of $D_{init}\cup D_0$.
\end{enumerate}

For any state $s$, any value assignment $e$ of $\sigma^{act}$ such that at most one action is true, and any value assignment $pf$ of $\sigma^{pf}$, we use $\phi(s, e, pf)$ to denote the state $s'$ such that $(\langle s, a, s^\prime\rangle, pf)$ is a pf-transition (According to Assumption 2, such $s^\prime$ must be unique). For any interpretation $I$, $i\in \{0, \dots, m\}$ and any subset $\sigma^\prime$ of $\sigma$, we use $I|_{i:\sigma^\prime}$ to denote the value assignment of $I$ to atoms in $i:\sigma^\prime$. Given any value assignment $TC$ of $0\!:\!\sigma^{initpf}\cup \sigma^{pf}_m$and a value assignment $A$ of $\sigma_m^{act}$, we construct an interpretation $I_{TC\cup A}$ of $Tr(D, m)$ that satisfies $TC \cup A$ as follows:
\begin{itemize}
\item  For all atoms $p$ in $\sigma^{pf}_m\cup 0\!:\!\sigma^{initpf}$, 
           we have $I_{TC\cup A}(p) = TC(p)$;
\item  For all atoms $p$ in $\sigma_m^{act}$, we have $I_{TC\cup A}(p) = A(p)$;
\item $(I_{TC\cup A})|_{0:\sigma^{fl}}$ is the assignment such that $(I_{TC\cup A})|_{0:\sigma^{fl}\cup 0:\sigma^{initpf}}$ is a stable model of $D_{init}\cup D_0$.
\item For each $i\in \{1, \dots, m\}$, $$(I_{TC\cup A})|_{i:\sigma^{fl}} = \phi((I_{TC\cup A})|_{(i-1):\sigma^{fl}}, (I_{TC\cup A})|_{(i-1):\sigma^{act}}, (I_{TC\cup A})|_{(i-1):\sigma^{pf}}).$$
\end{itemize}
By Assumptions 2 and 3, the above construction produces a unique interpretation. 

It can be seen that in the multi-valued probabilistic program $Tr(D, m)$ translated from $D$, the probabilistic constants  are $0\!:\!\sigma^{initpf}\cup \sigma^{pf}_m$. We thus call the value assignment of an interpretation $I$ on $0\!:\!\sigma^{initpf}\cup \sigma^{pf}_m$ the {\em total choice} of $I$. The following theorem asserts that the probability of a stable model under $Tr(D, m)$ can be computed by simply dividing the probability of the total choice associated with the stable model by the number of choice of actions.

\begin{thm}\label{thm:path-probability}
For any value assignment $TC$ of $ 0\!:\!\sigma^{initpf}\cup\sigma^{pf}_m$ and any value assignment $A$ of $\sigma_m^{act}$, there exists exactly one stable model $I_{TC\cup A}$ of $Tr(D, m)$ that satisfies $TC\cup A$, and the probability of $I_{TC\cup A}$ is
\[
Pr_{Tr(D, m)}(I_{TC\cup A}) = \frac{\underset{c=v\in TC}{\prod}M(c=v)}{(|\sigma^{act}| + 1)^{m}}.
\]
\end{thm}

The following theorem tells us that the conditional probability of transiting from a state $s$ to another state $s^\prime$ with action $e$ remains the same for all timesteps, i.e., the conditional probability of $i\!+\!1\!:\!s^\prime$ given $i:s$ and $i:e$ correctly represents the transition probability from $s$ to $s^\prime$ via $e$ in the transition system.

\begin{thm}\label{thm:transition-probability}
For any state $s$ and $s^\prime$, and action $e$, we have
\[
Pr_{Tr(D, m)}(i\!+\!1\!:\!s^\prime\mid i:s, i:e) = Pr_{Tr(D, m)}(j\!+\!1\!:\!s^\prime\mid j:s, j:e)
\]
for any $i, j\in\{0, \dots, m-1\}$ such that $Pr_{Tr(D, m)}(i:s)> 0$ and $Pr_{Tr(D, m)}(j:s)> 0$.
\end{thm}

For every subset $X_m$ of $\sigma_m\setminus\sigma^{pf}_m$, let $X^i(i < m)$ be the triple consisting of
\begin{itemize}
\item the set consisting of atoms $A$ such that $i:A$ belongs to $X_m$ and $A\in \sigma^{fl}$;
\item the set consisting of atoms $A$ such that $i:A$ belongs to $X_m$ and $A\in \sigma^{act}$;
\item the set consisting of atoms $A$ such that $i\!+\!1\!:\!A$ belongs to $X_m$ and $A\in \sigma^{fl}$.
\end{itemize}
Let $p(X^i)$ be the transition probability of $X^i$, $s_0$ is the interpretation of $\sigma^{fl}_0$ defined by $X^0$, and $e_i$ be the interpretations of $i:\sigma^{act}$ defined by $X^{i}$.

Since the transition probability remains the same, the probability of a path given a sequence of actions can be computed from the probabilities of transitions.

\begin{cor}\label{thm:reduce2transition}
For every $m\geq 1$, $X_m$ is a residual (probabilistic) stable model of $Tr(D, m)$ iff $X^0, \dots, X^{m-1}$ are transitions of $D$ and $0\!:\!s_0$ is a residual stable model of $D_{init}$. Furthermore, 
\[
Pr_{Tr(D, m)}(X_m\mid 0\!:\!e_0, \dots, m-1\!:\!e_{m-1}) = p(X^0)\times\dots\times p(X^m)\times Pr_{Tr(D, m)}(0\!:\!s_0).
\]
\end{cor}

\BOC

Like ${\cal BC}$ and ${\cal BC}$+, language $\pbcp$ assumes that a propositional signature $\sigma$ is constructed from ``constants'' and their ``values.'' 
A {\em constant} $c$ is a symbol that is associated with a finite set $\i{Dom}(c)$, called the {\em domain}. 
The signature $\sigma$ is constructed from a finite set of constants, consisting of atoms $c\!=\!v$
for every constant $c$ and every element $v$ in $\i{Dom}(c)$.
If the domain of~$c$ is $\{\false,\true\}$, then we say that~$c$ is {\em Boolean}, and abbreviate $c\mvis\true$ as $c$ and $c\mvis\false$ as~$\sneg c$. 

There are four types of constants in $\pbcp$: {\em fluent constants}, {\em action constants},  {\em pf (probability fact) constants} and {\em  initpf (initial probability fact) constants}. Fluent constants are further divided into {\em regular} and {\em statically determined}. The domain of every action constant is restricted to Boolean. An {\em action description} is a finite set of {\em causal laws}, which describes how fluents depend on each other statically and how their values change from one time step to another. Figure~\ref{fig:pbcplus-causal-laws} lists causal laws in $\pbcp$ and their translations into $\lpmln$.
A {\em fluent formula} is a formula such that all constants occurring in it are fluent constants. 

We use $\sigma^{fl}$ ($\sigma^{act}$, $\sigma^{pf}$, and $\sigma^{initpf}$, respectively)  to denote the set of all atoms $c=v$ where $c$ is a fluent constant (action constant, pf constant, initpf constant, respectively) of $\sigma$ and $v$ is in $\i{Dom}(c)$. 
For any subset $\sigma'$ of $\sigma$ and any $i\in\{0, \dots, m\}$, we use $i\!:\!\sigma^\prime$ to denote the set
$\{i\!:\!A \mid A\in\sigma^\prime\}$. For any formula $F$ of signature $\sigma$, by $i\!:\!F$ we denote the result of inserting $i\!:$ in front of every occurrence of every constant in~$F$. 

The semantics of a $\pbcp$ action description $D$ is defined by a translation into an $\lpmln$ program $Tr(D, m) = D_{init}\cup D_m$. Below we describe the essential part of the translation that turns a $\pbcp$ description into an $\lpmln$ program. 

The signature $\sigma_m$ of $D_m$ consists of atoms of the form $i\!:\!c=v$ such that
\begin{itemize}
\item for each fluent constant $c$ of $D$, $i\in\{0, \dots, m\}$ and $v\in Dom(c)$,
\item for each action constant or pf constant $c$ of $D$, $i\in\{0, \dots, m-1\}$ and $v\in Dom(c)$.
\end{itemize}

\BOCC
For $i\in\{0, \dots, m\}$, we use $i:\sigma^{x}_m$ to denote the subset of $\sigma_m^{x}$:
$$
\{i\!:\!c=v\ \mid\  i:c=v\in \sigma_m^{x}\}.
$$
\EOCC

\begin{figure}[t]
\centering
 \includegraphics[width=1\textwidth]{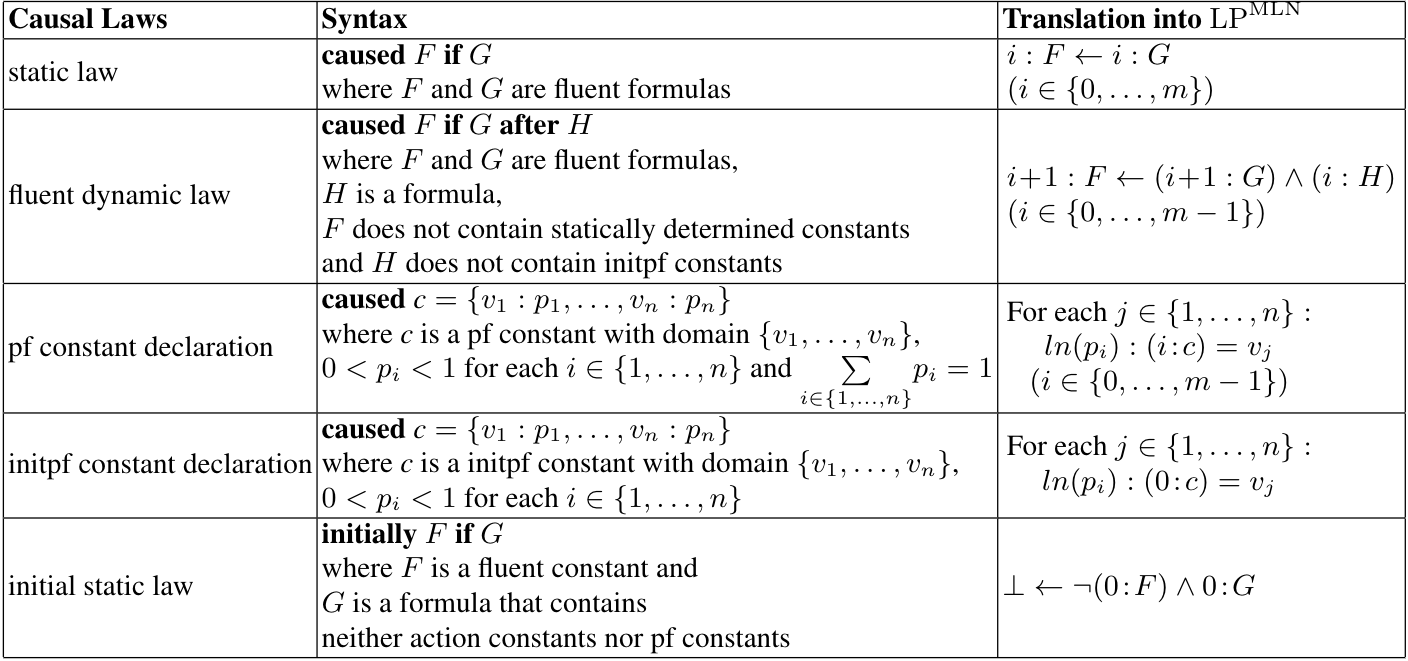}
\caption{Causal laws in $\pbcp$ and their translations into $\lpmln$}
\label{fig:pbcplus-causal-laws}
\end{figure}

$D_m$ contains $\lpmln$ rules obtained from static laws, fluent dynamic laws, and pf constant declarations as described in the third column of Figure \ref{fig:pbcplus-causal-laws}, as well as   $\{0\!:\!c=v\}^{\rm ch}$ for every regular fluent constant $c$ and every $v\in Dom(c)$, and
$\{i\!:\!c=\true\}^{\rm ch}, \{i\!:\!c=\false\}^{\rm ch}$ ($i\in\{0,\dots,m\!-\!1$)
for every action constant $c$ to state that the fluents at time $0$ and the actions at each time are exogenous.\footnote{$\{A\}^{\rm ch}$ denotes the choice rule $A\leftarrow \no\ \no\ A$.}
$D_{init}$ contains $\lpmln$ rules obtained from initial static laws and initpf constant declarations as described in the third column of Figure~\ref{fig:pbcplus-causal-laws}.
Both $D_m$ and $D_{init}$ also contain constraints asserting that each constant is mapped to exactly one value in its domain. In the presence of these constraints, we identify an interpretation of $\sigma_m$ with the value assignment function that maps each constant to its value.

\BOCC
\begin{table}[htb]
\centering
\begin{tabular}{|l|l|l|}
\hline
{\bf Causal Laws} & {\bf Syntax} & {\bf Translation into $\lpmln$} \\
\hline
    static law &
    \begin{tabular}[c]{@{}l@{}}
    $\caused\ F\ \iif\ G$\\
    where $F$ and $G$ are fluent formulas
    \end{tabular}  &
    $\begin{array}{l}
    i:F \leftarrow i:G\\
    (i\in\{0, \dots, m\})
    \end{array}$   \\
    \hline
    fluent dynamic law &
    \begin{tabular}[c]{@{}l@{}}
    $\caused\ F\ \iif\ G\ \after\ H$\\
    \begin{tabular}[c]{@{}l@{}}
    where $F$ and $G$ are fluent formulas,\\
    $H$ is a formula, \\
    $F$ does not contain statically determined constants\\
    and $H$ does not contain initpf constants\\
     \end{tabular}
    \end{tabular}  &
    $\begin{array}{l}
    i\!+\!1:F\leftarrow (i\!+\!1:G)\wedge(i:H)\\
    (i\in\{0, \dots, m-1\})
    \end{array}$\\
    \hline
    pf constant declaration &
    \begin{tabular}[c]{@{}l@{}}
    $\caused\ \i{c}=\{v_1:p_1, \dots, v_n:p_n\}$\\
    \begin{tabular}[c]{@{}l@{}}
    where $\i{c}$ is a pf constant with domain $\{v_1, \dots, v_n\}$,\\
    $0<p_i<1$ for each $i\in\{1, \dots, n\}$ and $\underset{i\in\{1, \dots, n\}}{\sum}p_i=1$
    \end{tabular}
    \end{tabular}&
    $\begin{array}{c}
    \text{For each $j\in\{1, \dots, n\}$}:\\
    ln(p_i):(i\!:\!c)=v_j  \\
    (i\in\{0, \dots, m-1\}) 
    \end{array}$\\
    \hline
    initpf constant declaration &
    \begin{tabular}[c]{@{}l@{}}
    $\caused\ \i{c}=\{v_1:p_1, \dots, v_n:p_n\}$\\
    \begin{tabular}[c]{@{}l@{}}
    where $\i{c}$ is a initpf constant with domain $\{v_1, \dots, v_n\}$,\\
    $0<p_i<1$ for each $i\in\{1, \dots, n\}$
    \end{tabular}
    \end{tabular}&
     $\begin{array}{c}
    \text{For each $j\in\{1, \dots, n\}$}:\\
    ln(p_i):(0\!:\!c)=v_j \\
    \end{array}$\\
    \hline
    initial static law &
    \begin{tabular}[c]{@{}l@{}}
     $\init\ F\ \iif\ G$\\
    \begin{tabular}[c]{@{}l@{}}
    where $F$ is a fluent constant and\\
    $G$ is a formula that contains \\
    neither action constants nor pf constants
    \end{tabular}
    \end{tabular}&
   $\bot\leftarrow \neg(0\!:\!F)\land 0\!:\!G$
    \\
    \hline
\end{tabular}
\caption{Causal Laws in $\pbcp$}
\label{tab:pbcplus-causal-laws}
\end{table}
\EOCC




%
For any $\lpmln$ program $\Pi$ of signature $\sigma_1$ and an interpretation $I$ of a subset $\sigma_2$ of~$\sigma_1$, we say $I$ is a {\em residual (probabilistic) stable model} of $\Pi$ if there exists an interpretation $J$ of $\sigma_1\setminus \sigma_2$ such that $I\cup J$ is a (probabilistic) stable model of $\Pi$.

For any interpretation $I$ of $\sigma$, by $i\!:\!I$ we denote the interpretation of  $i\!:\!\sigma$ such that $i\!:\!I\models (i\!:\!c)=v$ iff $I\models c=v$. 
For $x\in \{act, fl, pf\}$, we use $\sigma^{x}_m$ to denote the subset of $\sigma_m$, which is 
$
\{i\!:\!c=v\in\sigma_m \mid\ \text{$c=v\in\sigma^{x}$}\}.
$

A {\em state} of $D$ is an interpretation $I^{fl}$ of $\sigma^{fl}$ such that $0\!:\!I^{fl}$ is a residual (probabilistic) stable model of $D_0$. A {\em transition} of $D$ is a triple $\langle s, e, s^\prime\rangle$  where $s$ and $s^\prime$ are interpretations of $\sigma^{fl}$ and $e$ is an interpretation of $\sigma^{act}$ such that $0\!:\!s \cup 0\!:\!e \cup 1:s^\prime$ is a residual stable model of $D_1$. A {\em pf-transition} of $D$ is a pair $(\langle s, e, s^\prime\rangle, pf)$, where $pf$ is a value assignment to $\sigma^{pf}$ such that $0\!:\!s\cup 0\!:\!e \cup 1:s^\prime \cup 0\!:\!pf$ is a stable model of $D_1$.


\BOCC
\begin{definition}
A {\em probabilistic transition system} $T(D)$ represented by a probabilistic action description $D$ is a labeled directed graph such that the vertices are the states of $D$, and the edges are obtained from the transitions of $D$: for every transition $\langle s, e, s^\prime\rangle$  of $D$, an edge labeled $e: p$ goes from $s$ to $s^\prime$, where $p=Pr_{D_m}(1\!:\!s^\prime \mid 0\!:\!s, 0\!:\!e)$. The number $p$ is called the {\em transition probability} of $\langle s, e ,s^\prime\rangle$ .
\end{definition}
\EOCC

\BOCC
The soundness of the definition of a probabilistic transition system relies on the following proposition. 
\begin{prop}\label{prop:state-in-transition}
For any transition $\langle s, e, s^\prime \rangle$, $s$ and $s^\prime$ are states.
\end{prop}
\EOCC

The following simplifying assumptions are made on action descriptions in $\pbcp$.
%
\begin{enumerate}
\item {\bf No Concurrency}: For all transitions $\langle s, e, s'\rangle$, we have $e\models a\!=\!\true$ for at most one action constant $a$; 
%
\item {\bf Nondeterministic Transitions are Determined by pf constants}: For any state $s$, any value assignment $e$ of $\sigma^{act}$, and any value assignment $pf$ of $\sigma^{pf}$, there exists exactly one state $s^\prime$ such that $(\langle s, e, s^\prime\rangle, pf)$ is a pf-transition;
\item {\bf Nondeterminism on Initial States are Determined by Initpf constants}: For any value assignment $pf_{init}$ of $\sigma^{initpf}$, there {exists exactly one value assignment $fl$ of $\sigma^{fl}$ such that $0\!:\!pf_{init}\cup 0\!:\!fl$ is a stable model of $D_{init}\cup D_0$.}


\end{enumerate}

With the above three assumptions, the probability of a history, i.e., a sequence of states and actions, can be computed as the product of the probabilities of all the transitions that the history is composed of,  multiplied by the probability of the initial state (Corollary 1 in \cite{lee18aprobabilistic}).

\BOCC
\begin{cor}\label{thm:reduce2transition}
For every $m\geq 1$, $X_m$ is a residual (probabilistic) stable model of $Tr(D, m)$ iff $X^0, \dots, X^{m-1}$ are transitions of $D$ and $0\!:\!s_0$ is a residual stable model of $D_{init}$. Furthermore, 
\[
Pr_{Tr(D, m)}(X_m\mid 0\!:\!e_0, \dots, m-1\!:\!e_{m-1}) = p(X^0)\times\dots\times p(X^m)\times Pr_{Tr(D, m)}(0\!:\!s_0).
\]
\end{cor}
\EOCC

\BOCC
\begin{example}
Consider the simple transition system with probabilistic effects in Example \ref{eg:simple-pt}. Suppose $a$ is executed twice. What is the probability that $P$ remains true the whole time? With Corollary \ref{thm:reduce2transition} this can be computed as follows:
\[
\ba l
\small 
Pr(2:P=\true, 1\!:\!P=\true, 0\!:\!P=\true\mid 0\!:\!A=\true, 1\!:\!A=\true)\\
 = p(\langle P=\true, A=\true, P=\true\rangle)\cdot p(\langle P=\true, A=\true, P=\true\rangle)\cdot Pr_{Tr(D, m)}(0\!:\!P=\true)\\
 =\  0.2\times 0.2 \times 0.6 = 0.024.
\ea 
\]
\end{example}
\EOCC
\EOC


\subsection{Review: Markov Decision Process}

\begin{definition}
A {\em Markov Decision Process (MDP)} $M$ is a tuple
$
\langle S, A, T, R \rangle
$
where 
(i) $S$ is a set of states; (ii) $A$ is a set of actions;
(iii) $T:S\times A\times S\rightarrow [0, 1]$ defines transition probabilities;
(iv) $R:S\times A\times S\rightarrow \mathbb{R}$ is the reward function.
\end{definition}




Given a history $\vec{h} = \langle s_0, a_0, s_1, \dots,s_{m-1}, a_{m-1}, s_m\rangle$ such that each $s_i\in S$ $(i\in\{0, \dots, m\})$ and each $a_i\in A$  $(i\in\{0, \dots, m-1\})$, the {\em total reward} $R_{M}$ of the history under MDP $M$ is defined as
\\[-1em]
\[
R_{M}(\vec{h}) = \sum_{i=0}^{m-1} R(s_i, a_i, s_{i+1}).
\]
The probability  $P_{M}$ of $\vec{h}$ under MDP is defined as
\\[-1em]
\[
P_{M}(\vec{h}) = \prod_{i=0}^{m-1} T(s_i, a_i, s_{i+1}).
\]
\NB{empty action?} 
A {\em non-stationary policy} $\pi:S\times ST\mapsto A$ is a function from $S\times ST$ to $A$, where $ST=\{0, \dots, m-1\}$. 
The {\em expected total reward} of a non-stationary policy $\pi$ starting from the initial state $s_0$ under MDP $M$ is
%
{\small 
\begin{align}
\nonumber \i{ER}_{M}(\pi, s_0) =\ &\underset{\substack{\langle s_1, \dots, s_m\rangle : \\\text{$s_i\in S$ for $i\in \{1, \dots, m\}$}}}{E}[R_{M}(\langle s_0, \pi(s_0, 0), s_1, \dots,s_{m-1}, \pi(s_{m-1}, m-1), s_m\rangle)]\\
\nonumber  =\ &\underset{\substack{\langle s_1, \dots, s_m\rangle: \\\text{$s_i\in S$ for $i\in \{1, \dots, m\}$}}}{\sum}  \Big(\sum_{i=0}^{m-1} R(s_i, \pi(s_i, i), s_{i+1})\Big) \times \Big(\prod_{i=0}^{m-1} T(s_i, \pi(s_i, i), s_{i+1})\Big).
\end{align}
}
The {\em finite horizon policy optimization} problem starting from $s_0$ is to find a non-stationary policy $\pi$ that maximizes its expected total reward starting from $s_0$, i.e.,
$
{\rm argmax}_\pi\ \i{ER}_{M}(\pi, s_0).
$

Various algorithms for MDP policy optimization have been developed, such as value iteration~\cite{bellman57markovian} for exact solutions, and 
Q-learning~\cite{watkins89learning} for approximate solutions. 

\BOCC
\vspace{-5mm}
\subsubsection{Infinite Horizon Policy Optimization} 
Policy optimization with the infinite horizon is defined similar to the finite horizon, except that a discount rate for the reward is introduced, and the policy is stationary, i.e., no need to mention time steps (ST). Given an infinite sequence of states and actions $\langle s_0, a_0, s_1, a_1, \dots\rangle$, such that each $s_i\in S$ and each $a_i\in A$ $(i\in\{0, \dots\})$, and a discount factor $\gamma\in [0, 1]$, the {\em discounted total reward} of the sequence under MDP $M$ is defined as
\[
R_{M}(\langle s_0, a_0, s_1, a_1, \dots\rangle) = \sum_{i=0}^{\infty} \gamma^{i+1} R(s_i, a_i, s_{i+1}).
\]

\EOCC 

\section{$\dtlpmln$: A Decision Theoretic Extension of $\lpmln$}\label{sec:dt-lpmln}

We extend the syntax and the semantics of $\lpmln$ to $\dtlpmln$ by introducing atoms of the form 
\begin{equation}\label{eq:utility-atoms}
{\tt utility}(u, {\bf t})
\end{equation}
where $u$ is a real number, and ${\bf t}$ is an arbitrary list of terms. These atoms are called {\em utility atoms}, and they can only occur in the head of hard rules of the form
\begin{equation}\label{eq:utility-rule}
\alpha: {\tt utility}(u, {\bf t}) \leftarrow \i{Body}
\end{equation}
where $\i{Body}$ is a list of literals. 
We call these rules {\em utility rules}. Allowing an arbitrary list of terms as arguments of a utility atom provides control over how to distribute utility value over ground instances of a utility rule. For example, the user can choose to assign utility value only once for all ground instances, by not including any terms occurring in the body, as in
\[
{\tt utility}(10) \leftarrow {\tt package\_delivered}(package\_id),
\]
which specifies that the agent obtains a utility of $10$ if at least one package has been delivered.
The user can also choose to assign utility value for each ground instance by including those terms, as in
\[
{\tt utility}(10, package\_id) \leftarrow {\tt package\_delivered}(package\_id),
\]
which specifies that the agent obtains a utility of $10$ for each package delivered.

The weight and the probability of an interpretation are defined the same as in $\lpmln$.

\begin{definition}
The {\em utility} of an interpretation $I$ under $\Pi$ is defined as
\[
U_{\Pi}(I) = \underset{{\tt utility}(u, {\bf t})\in I}{\sum} u .
\]
The {\em expected utility} of a proposition $A$ is defined as
\begin{equation}\label{eq:expected-utility}
E[U_{\Pi}(A)] = \underset{I\models A}{\sum}\  
    U_{\Pi}(I) \times P_{\Pi}(I\mid A) . 
\end{equation}
i.e, the sum of the utilities of all interpretations satisfying $A$, weighted by their probability given $A$.
\end{definition}

\BOCC
For any proposition $F$, let $F^c$ denote $F$ represented as a constraint. To compute the expected utility with MC-ASP, \eqref{eq:expected-utility}, we simply collect a set $S$ of sample stable models of $\Pi\cup \{A^c\}$ using MC-ASP, and approximate \eqref{eq:expected-utility} with ${\underset{J\in S}{\sum}U_{\Pi}(J)}/{|S|}$, where $U_\Pi(J)T$ is obtained from summing up the first argument of {\tt utility} facts occurring in the samples.
\EOCC

A $\dtlpmln$ program is a pair
$(\Pi, Dec)$
where $\Pi$ is an $\lpmln$ program with a propositional signature $\sigma$ (including utility atoms) and $Dec$ is a subset of $\sigma$ consisting of {\em decision atoms}. We consider two reasoning tasks with $\dtlpmln$.
\begin{itemize}
\item {\bf Evaluating a Decision.}\  
Given a propositional formula $e$ (``evidence'') and a truth assignment $dec$ of decision atoms $Dec$, represented as a conjunction of literals over atoms in $Dec$, compute the expected utility of decision $dec$ in the presence of evidence $e$, i.e., compute
\[
E[U_{\Pi}(dec\wedge e)] = \underset{I\models dec\wedge e}{\sum}\  U_{\Pi}(I) \times P_{\Pi}(I\mid dec\wedge e) .
\]
\item {\bf Finding a Decision with Maximum Expected Utility (MEU)}.\   Given a propositional formula $e$ (``evidence''),  find the truth assignment $dec$ on $Dec$ such that the expected utility of $dec$ in the presence of $e$ is maximized, i.e., compute
\begin{equation}\label{eq:meu}
\underset{dec\ :\ \text{$dec$ is a truth assignment on $Dec$}}{\rm \rm argmax} E[U_{\Pi}(dec\wedge e)] . 
\end{equation}
\end{itemize}

{\cred
Algorithm \ref{alg:max-walk-sat-meu} is an approximate algorithm based on MaxWalkSAT  \cite{kautz98general} for solving the MEU problem. For any truth assignment $X$ on a set $\sigma$ of atoms and an atom $v\in \sigma$, we use $X\!\!\mid_v$ to denote the truth assignment on $\sigma$ obtained from $X$ by flipping the truth value of $v$.

Similar to MaxWalkSAT, Algorithm \ref{alg:max-walk-sat-meu} starts with random truth assignments on atoms in $Dec$. At each iteration, the algorithm flips the truth value of an atom in $Dec$. It either chooses to flip a random atom in $Dec$ (with probability $p$), or an atom whose value flipping would result in a largest improvement on the expected utility (with probability $1-p$). The chance of random flipping is a way of getting out of local optima. The algorithm also performs the search process multiple times ($m_t$), each time starting from a different random truth assignment on $Dec$.
\begin{algorithm}[h!]
{\footnotesize
\noindent {\bf Input: }
\begin{enumerate}
\item $(\Pi, A)$: A DT-$\lpmln$ program;
\item $E$: a proposition in constraint form as the evidence;
\item $m_t$: the maximum number of tries;
\item $m_f$: the maximum number of flips;
\item $p$: probability of taking a random step.
\end{enumerate}
\noindent {\bf Output: } $soln$: a truth assignment on $A$

\noindent {\bf Process:}
\begin{enumerate}
\item $soln\leftarrow null$;
\item $utility \leftarrow -\infty$;
\item For $i\leftarrow 1$ to $m_t$:
\begin{enumerate}
\item $X\leftarrow$ a random soft stable model of $\Pi\cup E$;
\item $soln'\leftarrow$ truth assignment of $X$ on $A$;
\item $utility'\leftarrow E[U_\Pi(soln')]$;
\item For $j\leftarrow 1$ to $m_f$:
\begin{enumerate}
\item $flippable \leftarrow \{v \mid \text{$soln\!\!\mid_v$ is a partial stable model of $\Pi\cup E$}\}$;
\item For each atom $v$ in $flippable$:\\
\hspace{2mm} $DeltaCost(v)\leftarrow E[U_\Pi(soln'\wedge E)]-E[U_\Pi(soln'\!\!\mid_v\wedge E)]$;
\item If $Uniform(0,1)< p$:\\
\hspace{1cm} $v_f\leftarrow$ a randomly chosen decision atom in $flippable$;\\
\hspace{1mm} else:\\
\hspace{1cm}$v_f\leftarrow \underset{v\in flippable}{argmin}\ \ DeltaCost(v)$;
\item If $DeltaCost(v_f)<0$:
\begin{enumerate}
\item $soln'\leftarrow soln'\!\!\mid_{v_f}$;
\item $utility'\leftarrow utility' - DeltaCost(v_f)$.
\end{enumerate}
\end{enumerate}
\item If $utility'>utility$:
\begin{enumerate}
\item $utility\leftarrow utility'$;
\item $soln\leftarrow soln'$;
\end{enumerate}
\end{enumerate}
\item Return $soln$
\end{enumerate}
\caption{MaxWalkSAT for Maximizing Expected Utility}
\label{alg:max-walk-sat-meu}
}
\end{algorithm}

\begin{example}\label{eg:marketing}
Consider a directed graph $G$ representing a social network: 
(i)  each vertex $v\in V(G)$ represents a person;
each edge $(v_1, v_2)$ represents that $v_1$ influences $v_2$; 
(ii) each edge $e=(v_1, v_2)$ is associated with a probability $p_e$ representing the probability of the influence; 
(iii) each vertex $v$ is associated with a cost $c_v$, representing the cost of marketing the product to $v$;
(iv) each person who buys the product yields a reward of $r$.

The goal is to choose a subset $U$ of vertices as marketing targets so as to maximize the expected total profit. 
The problem can be represented as a $\dtlpmln$ program $\Pi^{\rm market}$ as follows:
\begin{align}
\nonumber \alpha: &\ buy(v) \leftarrow marketTo(v).\\
\nonumber \alpha: &\ buy(v_2) \leftarrow buy(v_1), influence(v_1, v_2).\\
\nonumber \alpha: &\ utility(r, v)\leftarrow buy(v).
\end{align}
with the graph instance represented as follows: 
\begin{itemize}
\item for each edge $e=(v_1, v_2)$, we introduce a probabilistic fact
$
ln(\frac{p_e}{1-p_e}): influence(v_1, v_2);
$
\item for each vertex $v\in V(G)$, we introduce the following rule:\\ 
$
 \alpha: {\tt utility}(-c_v, v)\leftarrow marketTo(v).
$
\end{itemize}

For simplicity, we assume that marketing to a person guarantees that the person buys the product. This assumption can be removed easily by changing the first rule to a soft rule.

The MEU solution of $\dtlpmln$ program $(\Pi^{\rm market}, \{marketTo(v)\mid v\in V(G)\})$ corresponds to the subset $U$ of vertices that maximizes the expected profit.

\NB{change cost}

\noindent\medskip
\begin{minipage}{.6\textwidth}
For example, consider the directed graph on the right, where each edge $e$ is labeled by $p_e$ and each vertex $v$ is labeled by $c_v$. Suppose the reward for each person buying the product is $10$. There are $2^6 = 64$ different truth assignments on decision atoms, corresponding to $64$ choices of marketing targets. The best decision is to market to {\tt Alice} only, which yields the expected utility of $17.96$. 
\end{minipage}
\begin{minipage}{.1\textwidth}
\end{minipage}
\begin{minipage}{.3\textwidth}
\includegraphics[scale=0.3]{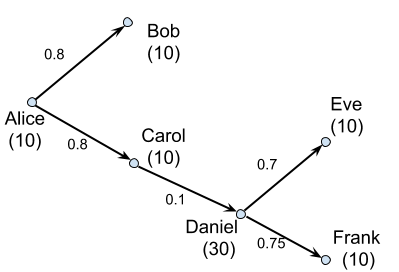}
\end{minipage}


\end{example}

We implemented Algorithm~\ref{alg:max-walk-sat-meu} and report in Figure~\ref{fig:meu-marketing-stats} its performance on the domain described in Example \ref{eg:marketing}. We generate networks with $10, 12, \dots, 20$ people and randomly generated edges, and use Algorithm~\ref{alg:max-walk-sat-meu} with MC-ASP as the underlying sampling methods for approximating expected utilities, to find the optimal set of marketing targets. We compare the performance of the algorithm with system {\sc DT-problog} \cite{broeck10dtproblog}. We use {\sc DT-problog} with exact mode and approximate mode respectively for the same task. The graphs contain directed cycles. For Algorithm~\ref{alg:max-walk-sat-meu}, $50$ stable models are sampled to approximate each expected utility, $p$ is set to be $0.5$, $m_t$ is set to be $10$, and $m_f$ is set to be $10$. The experiments were performed on a machine powered by 4 Intel(R) Core(TM) i5-2400 CPU with OS Ubuntu 14.04.5 LTS and 8G memory. As can be seen from the result, $\dtlpmln$ outperforms both approximate and exact solving mode of {\sc DT-problog} on this example. A possible reason is that {\sc DT-problog} has to convert
the input program, combined with the query, into weighted Boolean formulas, which
is expensive for non-tight programs\footnote{We say an $\lpmln$ program is {\em tight} if $\o{\Pi}$ is tight according to \cite{lee03a}, i.e., the positive dependency graph of $\o{\Pi}$ is acyclic.}.

\begin{figure}
 \begin{center}
    \includegraphics[width=0.98\textwidth]{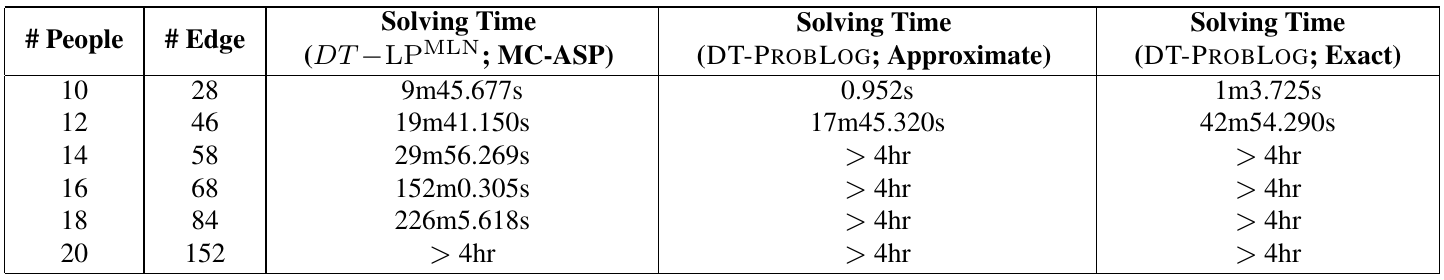}
    \end{center}
    \caption{Running Statistics of Algorithm~\ref{alg:max-walk-sat-meu} on Marketing Domain}
    \label{fig:meu-marketing-stats}
\end{figure}
}
\NB{Arrow too small?}

\section{$\pbcp$ with Utility}\label{sec:pbcplus-rewards}

We extend $\pbcp$ by introducing the following expression called {\em utility law} that assigns a reward to transitions:
\begin{equation}\label{eq:utility-law}
\reward\ v\ {\bf if}\ F\ {\bf after}\ G
\end{equation}
where $v$ is a real number representing the reward, $F$ is a formula that contains fluent constants only, and $G$ is a formula that contains fluent constants and action constants only (no pf, no initpf constants). We extend the signature of $Tr(D, m)$ with a set of atoms of the form~\eqref{eq:utility-atoms}. 
We turn a utility law of the form \eqref{eq:utility-law} into the $\lpmln$ rule
\begin{equation}\label{eq:utility-law-lpmln}
\alpha: {\tt utility}(v, i+1, id)\ \ar\ (i+1:F)\wedge(i:G)
\end{equation}
where $id$ is a unique number assigned to the $\lpmln$ rule and $i\in\{0,\dots, m\!-\!1\}$.

Given a nonnegative integer $m$ denoting the maximum time step, a $\pbcp$ action description $D$ with utility over multi-valued propositional signature $\sigma$ is defined as a high-level representation of the $\dtlpmln$ program $(Tr(D, m), \sigma^{act}_m)$.

%

We extend the definition of a probabilistic transition system as follows.
\begin{definition}
A {\em probabilistic transition system} $T(D)$ represented by a probabilistic action description $D$ is a labeled directed graph such that the vertices are the states of $D$, and the edges are obtained from the transitions of $D$: for every transition $\langle s, e, s^\prime\rangle$  of $D$, an edge labeled $e: p, u$ goes from $s$ to $s^\prime$, where $p=Pr_{D_1}(1\!:\!s^\prime \mid 0\!:\!s\wedge 0\!:\!e)$ and $u=E[U_{D_1}(0\!:\!s\wedge 0\!:\!e\wedge 1\!:\!s')]$. The number $p$ is called the {\em transition probability} of $\langle s, e ,s^\prime\rangle$, denoted by $p(s, e ,s^\prime)$, and the number $u$ is called the {\em transition reward} of $\langle s, e ,s^\prime\rangle$, denoted by $u(s, e ,s^\prime)$.
\end{definition}

\begin{example}\label{eg:simple}
{\sl 
The following action description $D^{simple}$ describes a simple probabilistic action domain with two Boolean fluents $P$, $Q$, and two actions $A$ and $B$. $A$ causes $P$ to be true with probability $0.8$, and if $P$ is true, then $B$ causes $Q$ to be true with probability $0.7$. The agent receives the reward $10$ if $P$ and $Q$ become true for the first time (after that, it remains in the state $\{P,Q\}$ as it is an absorbing state).

\begin{minipage}[c]{0.4\textwidth}
\[
\ba l
A\ \causes\ P\ \iif\ \i{Pf}_1\\
B\ \causes\ Q\ \iif\ P\wedge \i{Pf}_2\\
\inertial\ P, Q \\
\constraint\ \neg(Q\wedge\sneg P)\\
\caused\ \i{Pf}_1=\{\true: 0.8, \false: 0.2\}\\
\caused\ \i{Pf}_2=\{\true: 0.7, \false: 0.3\}
\ea 
\]
\end{minipage}
\begin{minipage}[c]{0.1\textwidth}
~~~
\end{minipage}
\begin{minipage}[c]{0.45\textwidth}
\[
\ba l
\reward\ 10\ \iif\ P\wedge Q\ \after\ \neg(P\wedge Q)\\
\caused\ \i{InitP}=\{\true: 0.6, \false: 0.4\}\\
\init\ P=x\ \iif\ \i{InitP}=x\\
\caused\ \i{InitQ}=\{\true: 0.5, \false: 0.5\}\\
\init\ Q\ \iif\ \i{InitQ}\wedge P\\
\init\ \sneg Q\ \iif\ \sneg P.
\ea 
\]
\end{minipage}

\vspace{0.3cm}

The transition system $T(D^{simple})$ is as follows:

 \begin{center}
    \includegraphics[height=2.3cm]{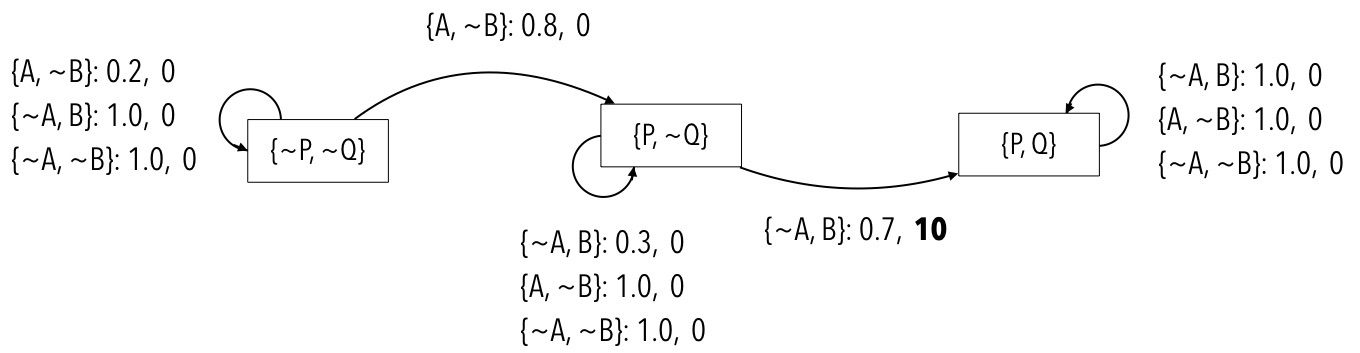}
    \end{center}
}
\end{example}
\vspace{-0.5cm}

\NBB{Remove figure (b)}

\BOCC
For every subset $X_m$ of
\[
(\sigma_m\cup \{{\tt utility}(v, {\bf x})\mid \text{${\tt utility}(v, {\bf x})$ occurs in $Tr(D, m)$}\})\setminus\sigma^{pf}_m, 
\]
let $X^i(i < m)$ be the triple consisting of
\begin{itemize}
\item the set consisting of atoms $A$ such that $i:A$ belongs to $X_m$ and $A\in \sigma^{fl}$;
\item the set consisting of atoms $A$ such that $i:A$ belongs to $X_m$ and $A\in \sigma^{act}$;
\item the set consisting of atoms $A$ such that $i\!+\!1\!:\!A$ belongs to $X_m$ and $A\in \sigma^{fl}$.
\end{itemize}
\EOCC

\BOCC
\begin{example}\label{eg:yale-shooting-utility}
Consider the following probabilistic variation of the well-known Yale Shooting Problem: There are two (deaf) turkeys: a fat turkey and a slim turkey. Shooting at a turkey may fail to kill the turkey. Normally, shooting at the slim turkey has $0.6$ chance to kill it, and shooting at the fat turkey has $0.9$ chance. However, when a turkey is dead, there is $50\%$ chance that the other turkey run away. Killing the fat turkey yields a reward of $8$, the slim turkey yields a reward of $10$.

The example can be modeled in $\pbcp$ as follows. We call the following action description $D^{turkey}$. First, we declare  the constants:

\vspace{0.2cm}
\hrule
\begin{tabbing}
Notation:  $t$ range over $\{\i{SlimTurkey}, \i{FatTurkey}\}$. \\
Regular fluent constants:          \hskip 4cm  \=Domains:\\
$\;\;\;$ $\i{Alive}(t)$,  $\i{Escaped}(t)$, $\;$ $\i{Loaded}$                 \>$\;\;\;$ Boolean\\ 
Action constants:                          \>Domains:\\
$\;\;\;$ $\i{Load}$ , $\;$ $\i{Fire}(t)$  \>$\;\;\;$ Boolean\\ 
Pf constants:                          \>Domains:\\
$\;\;\;$ $\i{\i{Pf\_Killed}}(t)$, $\;$ $\i{\i{Pf\_Escape}}(t)$                    \>$\;\;\;$ Boolean \\
InitPf constants: \\
$\;\;\;$ $\i{Init\_Alive}(t)$,  $\;$ $\i{Init\_Loaded}$                 \>$\;\;\;$ Boolean
\end{tabbing}
\hrule
\vspace{0.2cm}

Next, we state the causal laws.
The effect of loading the gun is described by 
\begin{tabbing}
\ \ \ $\caused\ \i{Loaded}\ \iif\ \top\ \after\ \i{Load}$.
\end{tabbing}
To describe the effect of shooting at a turkey, we declare the following probability distributions on the result of shooting at each turkey, respectively: 
\begin{tabbing}
\ \ \ $\caused\ \i{Pf\_Killed}(\i{SlimTurkey})=\{\true: 0.6, \false: 0.4\}$,  \\ 
\ \ \ $\caused\ \i{Pf\_Killed}(\i{FatTurkey})=\{\true: 0.9, \false: 0.1\}$.
\end{tabbing}
The effect of shooting at a turkey is described as
\begin{tabbing}
\ \ \ $\caused\ \sneg \i{Alive}(t)\ \iif\ \top\ \after\ \i{Loaded}\wedge \i{Fire}(t)\wedge \sneg \i{Escaped}(t)\wedge \i{Pf\_Killed}(t)$,\\
\ \ \ $\caused\ \sneg \i{Loaded}\ \iif\ \top\ \after\ \i{Fire}(t)$.
\end{tabbing}
A dead turkey causes the other turkey to escape with $0.5$ chance:
\begin{tabbing}
\ \ \ $\caused\ \i{Pf\_Escape}=\{\true: 0.5, \false: 0.5\}$,  \\
\ \ \ $\caused\ \i{Escaped}(t_1)\ \iif\ \top\ \after\ \sneg \i{Alive}(t_2)\wedge \i{Alive}(t_1) \wedge t_1\neq t_2 \wedge \i{Pf\_Escape} \wedge \sim\i{Escape}d(t_1)$.
\end{tabbing}
($\default\ F$ stands for $\caused\ \{F\}^{\rm ch}$ \cite{babb15action1}).
A dead turkey cannot be escaped:
\begin{tabbing}
\ \ \ $\caused\ \bot\ \iif\ \sim\i{Alive}(t)\wedge \i{Escaped}(t)$.
\end{tabbing}
In general, state constraints like the above static law could cause the action description to violate the assumptions we made, since pruning out a state is indirectly pruning out the actions that lead to the state, thus making certain actions not executable. However, the above static law does not cause violation of the assumptions, since no actions can lead to violation of the static law.
The fluents $\i{Escaped}$, $\i{Alive}$ and $\i{Loaded}$ observe the commonsense law of inertia:
\begin{tabbing}
\ \ \ $\caused\ \{\i{Escaped}(t)\}^{\rm ch}\ \iif\ \top\ \after\ \i{Escaped}(t)$, \\
\ \ \ $\caused\ \{\sneg \i{Escaped}(t)\}^{\rm ch}\ \iif\ \top\ \after\ \sneg \i{Escaped}(t)$, \\
\ \ \ $\caused\ \{\i{Alive}(t)\}^{\rm ch}\ \iif\ \top\ \after\ \i{Alive}(t)$, \\
\ \ \ $\caused\ \{\sneg \i{Alive}(t)\}^{\rm ch}\ \iif\ \top\ \after\ \sneg \i{Alive}(t)$, \\
\ \ \ $\caused\ \{\i{Loaded}\}^{\rm ch}\ \iif\ \top\ \after\ \i{Loaded}$,\ \ \\
\ \ \ $\caused\ \{\sneg \i{Loaded}\}^{\rm ch}\ \iif\ \top\ \after\ \sneg \i{Loaded}$.
\end{tabbing}
We ensure no concurrent actions are allowed by stating
\begin{tabbing}
\ \ \ $\caused\ \bot\ \after\ a_1\wedge a_2$
\end{tabbing}
for every pair of action constants $a_1, a_2$ such that $a_1\neq a_2$.

Killing the fat turkey and slim turkey yields a reward of $10$ and $8$, resp.
\begin{tabbing}
\ \ \ ${\bf reward}\ 8\ {\bf if}\ \sim\i{Alive}(FatTurkey)\ {\bf after}\ \i{Alive}(FatTurkey)$\\
\ \ \ ${\bf reward}\ 10\ {\bf if}\ \sim\i{Alive}(SlimTurkey)\ {\bf after}\ \i{Alive}(SlimTurkey)$
\end{tabbing}

Finally, we state that the initial values of all fluents are uniformly random ($b$ is a schematic variable that ranges over $\{\true, \false\}$):
\begin{tabbing}
\ \ \ $\caused\ \i{Init\_Alive}(t)=\{\true: 0.5, \false: 0.5\}$,  \\
\ \ \ $\caused\ \i{Init\_Loaded}=\{\true: 0.5, \false: 0.5\}$, \\
\ \ \ $\caused\ \i{Init\_Escaped(t)}=\{\true: 0.5, \false: 0.5\}$, \\
\ \ \ $\init\ \i{Escaped}(t)=b\ \iif\ \i{Init\_Escaped}(t)=b$,  \\
\ \ \ $\init\ \i{Alive}(t)=b\ \iif\ \i{Init\_Alive}(t)=b$,  \\
\ \ \ $\init\ \i{Loaded}=b\ \iif\ \i{Init\_Loaded}=b$.
\end{tabbing}

It can be seen that trying to kill the fat turkey yields better expected reward ($8\times 0.9 = 0.72$), compared to trying to kill the slim turkey ($10\times 0.6 = 6$). Informally, the following policy would be yield the highest expected total reward, given that initially both turkeys are alive, and the gun is not loaded: the agent should load the gun if it's not loaded; if the gun is loaded and the fat turkey is alive, then the agent should shoot at the fat turkey; otherwise, the agent should shoot at the slim turkey.
\end{example}
\EOCC

\BOCC
Various reasoning tasks on a $\pbcp$ action description involving plans can be viewed as decision evaluation or MEU problem on its underlying $\dtlpmln$ program. A plan can take many different forms. Here, for simplicity, we define a {\em plan} as a sequence of actions $a_1, a_2, \dots, a_{m-1}$.

\begin{itemize}
\item {\bf Plan Evaluation} Given some observation $obs$, represented as a propositional formula, and a plan $a_1, a_2, \dots, a_{m-1}$, we compute its expected utility, i.e.,
\[
E[U_{Tr(\Pi, m)}(a_1, \dots, a_{m-1}\wedge obs)].
\]
\item {\bf Planning} Given some observation $obs$, we compute the plan with highest expected utility, i.e.,
\[
\underset{plan: \text{$plan$ is a truth assignment on $\sigma^{act}_m$}}{\rm argmax}E[U_{Tr(\Pi, m)}(a_1, \dots, a_{m-1}\wedge obs)].
\]
\end{itemize}
\EOCC

\subsection{Policy Optimization}
Given a $\pbcp$ action description $D$, we use ${\bf S}$ to denote the set of states, i.e, the set of interpretations $I^{fl}$ of $\sigma^{fl}$ such that $0\!:\!I^{fl}$ is a residual (probabilistic) stable model of $D_0$. We use ${\bf A}$ to denote the set of interpretations $I^{act}$ of $\sigma^{act}$ such that $0\!:\!I^{act}$ is a residual (probabilistic) stable model of $D_1$. Since we assume at most one action is executed each time step, each element in ${\bf A}$ makes either only one action 
or none to be true. 

A {\em (non-stationary) policy} $\pi$ (in $\pbcp$) is a function
$
\pi: {\bf S}\times \{0, \dots, m-1\}\mapsto {\bf A}
$
that maps a state and a time step to an action (including doing nothing).
By 
$\langle s_{0}, s_{1}\dots,  s_{m} \rangle^t$  (each $s_{i} \in {\bf S}$) we denote 
the formula 
$0\!:\!s_{0}\wedge 1\!:\!s_{1}\wedge \dots\wedge  m\!:\!s_{m}$, 
and by \\
$\langle s_{0}, a_{0}, s_{1}\dots,  s_{m-1}, a_{m-1}, s_{m} \rangle^t$
(each $s_{i} \in {\bf S}$ and each $a_i\in {\bf A}$)
the formula $$0\!:\!s_{0}\wedge 0\!:\!a_{0}\wedge 1\!:\!s_{1}\wedge\dots \wedge m-1\!:\!a_{m-1}\wedge m\!:\!s_{m}.$$ 
\BOCC
by $i\!:\!s$, we denote the conjunction of all atoms that are true in 
we write $i\!:\!s$ as an {\cred abbreviation of the formula $\underset{fl\in \sigma^{fl}}{\bigwedge}i\!:\!fl = s(fl)$; }for any $i\in \{0,  \dots, m-1\}$ and $a\in {\bf A}$, we write $i\!:\!a$ as an abbreviation of the formula {\cred $\underset{act\in \sigma^{act}}{\bigwedge}i\!:\!act = a(act)$.}
\EOCC

We say a state $s$ is {\em consistent} with $D_{init}$ if there exists at least one probabilistic stable model $I$ of $D_{init}$ such that $I\models 0\!:\!s$. 
\begin{definition}
The {\em Policy Optimization} problem from the initial state $s_0$ is to find a policy $\pi$ that maximizes the expected utility starting from $s_0$, i.e., $\pi$ with 
\[
\underset{\text{$\pi$ is a policy}}
  {\rm argmax}\ E[U_{Tr(\Pi, m)}(C_{\pi, m}\wedge \langle s_0\rangle^t)]
\]
where $C_{\pi, m}$ is {the following formula representing policy $\pi$:
\[
\underset{s\in {\bf S},\ \pi(s, i)=a,\ i\in\{0, \dots, m-1\}}{\bigwedge} i\!:\!s\rar i\!:\!a\ .
\]
}
\end{definition}

We define the {\em total reward} of a history $\vec{h} = \langle s_0, a_0, s_1, \dots, s_m \rangle$ {\em under the action description } $D$
as 
\[
   R_D(\vec{h}) = E[U_{Tr(D, m)}(\vec{h}^t)].
\]

Although it is defined as an expectation, the following proposition tells us that any stable model $X$ of $Tr(D, m)$ such that $X\models \vec{h}$ has the same utility, and consequently, the expected utility of $\vec{h}$ is the same as the utility of any single stable model that satisfies the history. 

\begin{prop}\label{prop:history-determines-utility}
For any two stable models $X_1, X_2$ of $Tr(D, m)$ that satisfy a history \\ $\vec{h} = \langle s_0, a_0, s_1, a_1, \dots,  a_{m-1}, s_m\rangle$, we have
\begin{align}
\nonumber &U_{Tr(D, m)}(X_1)\ = \  U_{Tr(D, m)}(X_2)  
\ = \
E[U_{Tr(D, m)}(\vec{h}^t)].
\end{align}
\end{prop}

\BOC
We thus define the {\em total reward} of a history $\langle s_0, a_0, s_1, \dots, s_m \rangle$ {\em under action description } $D$, denoted by $R_D(\langle s_0, a_0, s_1, \dots, s_m \rangle)$, as the utility of any stable model $X$ of $Tr(D, m)$ that satisfies $\langle s_0, a_0, s_1, \dots, s_m \rangle$, i.e., 
\[
   R_D(\langle s_0, a_0, s_1, \dots, s_m \rangle) = U_{Tr(D,m)}(X).
\]
\EOC

\BOC
The {\em expected utility of a policy $\pi$ for a sequence of states $\langle s_0, s_1, \dots, s_m\rangle$} is defined as
\begin{align*}
 &U^D_\pi(\langle s_0, s_1, \dots, s_m\rangle) \\
 & \ \ \ \ = R_D(\langle s_0, \pi(s_0), s_1, \dots,\pi(s_{m-1}), s_m\rangle)
   \times P_{Tr(D, m)}(\langle s_0, s_1, \dots, s_m\rangle \mid \langle s_0\rangle\wedge C_{\pi, m}).
\end{align*}
\EOC

It can be seen that the expected utility of $\pi$ can be computed from the expected utility from all possible state sequences. For any state sequence $\vec{s} =\langle s_0, \dots, s_m\rangle$ and any policy $\pi$, we use $\vec{h}_\pi(\vec{s})$ to denote the history obtained by applying $\pi$ on $\vec{s}$, i.e., $$\vec{h}_\pi(\vec{s})=\langle s_0, \pi(s_0,  0), s_1, \dots,\pi(s_{m-1}, m-1), s_m \rangle.$$

\begin{prop}\label{prop:expected-utility-policy}
Given any initial state $s_0$ that is consistent with $D_{init}$, for any policy $\pi$, we have
\[
\ba l
\small 
E[U_{Tr(D, m)}(C_{\pi, m}\wedge \langle s_0\rangle^t)] =  \\
   \underset{\vec{s} = \langle s_1, \dots, s_m\rangle: s_i\in {\bf S}}{\sum}
      R_D(\vec{h}_\pi(\vec{s})^t)
      \times P_{Tr(D, m)}(\vec{s}^t\wedge C_{\pi, m}).
\ea 
\]
\end{prop}
\BOCC
\begin{proof}
We have
{\footnotesize
\begin{align}
\nonumber &E[U_{Tr(D, m)}(C_\pi\wedge 0\!:\!s_0)]\\
\nonumber =\ & \underset{I\models 0:s_0\wedge C_\pi}{\sum} P_{Tr(D, m)}(I\mid 0\!:\!s_0\wedge C_\pi)\cdot U_{Tr(D, m)}(I)\\
\nonumber =\ & \underset{\substack{I\models 0:s_0\wedge C_\pi\\\text{$I$ is a stable model of $Tr(D, m)$}}}{\sum} P_{Tr(D, m)}(I\mid 0\!:\!s_0\wedge C_\pi)\cdot U_{Tr(D, m)}(I)\\
\nonumber =\ & \text{(We partition stable models $I$ according to their truth assignment on $\sigma^{fl}_m$)}\\
\nonumber \ & \underset{\langle s_1, \dots, s_m\rangle: s_i\in {\bf S}\ \ }{\sum} \underset{\substack{I\models \langle 0:s_0, 1:s_1, \dots, m:s_m\rangle\wedge C_\pi\\ \text{$I$ is a stable model of $Tr(D, m)$}}}{\sum} P_{Tr(D, m)}(I\mid 0\!:\!s_0\wedge C_\pi)\cdot U_{Tr(D, m)}(I) \\
\nonumber =\ & \text{(Since $I\models \langle 0:s_0, 1:s_1, \dots, m:s_m\rangle\wedge C_\pi$ implies $I\models \langle 0\!:\!s_0, 0\!:\!\pi(s_0, 0), 1\!:\!s_1, \dots, m\!:\!s_m\rangle$, by Proposition \ref{prop:history-determines-utility} we have)}\\
\nonumber \ & \underset{\langle s_1, \dots, s_m\rangle: s_i\in {\bf S}\ \ }{\sum} \underset{\substack{I\models \langle 0:s_0, 1:s_1, \dots, m:s_m\rangle\wedge C_\pi\\ \text{$I$ is a stable model of $Tr(D, m)$}}}{\sum} P_{Tr(D, m)}(I\mid 0\!:\!s_0\wedge C_\pi)\cdot E[U_{Tr(D, m)}(\langle 0\!:\!s_0, 0\!:\!\pi(s_0, 0), 1\!:\!s_1, \dots, m\!:\!s_m\rangle)] \\
\nonumber =\ & \underset{\langle s_1, \dots, s_m\rangle: s_i\in {\bf S}}{\sum}Pr_{Tr(D, m)}(\langle 0\!:\!s_0, 1\!:\!s_1, \dots, m\!:\!s_m\rangle \mid 0\!:\!s_0\wedge C_\pi)\cdot E[U_{Tr(D, m)}(\langle 0\!:\!s_0, 0\!:\!\pi(s_0, 0), 1\!:\!s_1, \dots, m\!:\!s_m\rangle)]\\
\nonumber =\ & \underset{\langle s_1, \dots, s_m\rangle: s_i\in {\bf S}}{\sum}Pr_{Tr(D, m)}(\langle 0\!:\!s_0, 1\!:\!s_1, \dots, m\!:\!s_m\rangle \mid 0\!:\!s_0\wedge C_\pi)\cdot E[U_{Tr(D, m)}(\langle 0\!:s_0, 1\!:\!s_1, \dots, m\!:\!s_m\rangle\wedge C_\pi)]\\
\nonumber =\ & \underset{\langle s_1, \dots, s_m\rangle: s_i\in {\bf S}}{\sum}U^D_\pi(\langle 0\!:\!s_0, 1\!:\!s_1, \dots, m\!:\!s_m\rangle\mid 0\!:\!s_0).
\end{align}
}
\qed
\end{proof}
\EOCC

\begin{definition}\label{def:pbcp-mdp}
For a $\pbcp$ action description $D$, let $M(D)$ be the MDP
$
\langle S, A, T, R\rangle
$
where
(i) the state set $S$ is ${\bf S}$;
(ii) the action set $A$ is ${\bf A}$;
(iii) transition probability function $T$ is defined as $T(s, a, s')= P_{D_1}(1:s'\mid 0:s\wedge 0:a)$;
(iv) reward function $R$ is defined as $R(s, a, s') = E[U_{D_1}(0:s\wedge 0:a\wedge 1:s')]$.
\end{definition}

We show that the policy optimization problem for a $\pbcp$ action description $D$ can be reduced to the policy optimization problem for $M(D)$ for the finite horizon.
The following theorem 
tells us that for any history following a non-stationary policy, its total reward and probability under $D$ defined under the $\pbcp$ semantics coincide with those under the corresponding MDP $M(D)$.

\begin{thm}\label{thm:sequence-utility-equivalence}
Given an initial state $s_0\in {\bf S}$ that is consistent with $D_{init}$, for any non-stationary policy $\pi$ and any finite state sequence $\vec{s} = \langle s_0, s_1, \dots, s_{m-1}, s_m\rangle $ such that each $s_i$ in ${\bf S}\ (i\in \{0, \dots, m\})$, we have
\begin{itemize}
    \item $R_D(\vec{h}_\pi(\vec{s}))
           = R_{M(D)}(\vec{h}_\pi(\vec{s}))$
    \item $P_{Tr(D, m)}(\vec{s}^t \mid \langle s_0\rangle^t\wedge C_{\pi, m})
           = P_{M(D)}(\vec{h}_\pi(\vec{s}))$.
\end{itemize}
\end{thm}
\BOCC
\begin{align}
\nonumber & U^D_\pi(\langle s_0, s_1, \dots, s_m\rangle)  \\
\nonumber & \ \ \ \ \ \ = R_{M(D)}(\langle s_0, \pi(s_0, )\dots, \pi(s_{m-1}), s_m\rangle)\times P_{M(D)}(\langle s_0, \pi(s_0, )\dots, \pi(s_{m-1}), s_m\rangle).
\end{align}
\EOCC 

\BOCC
\begin{proof}
We have
\begin{align}
\nonumber &U^D_\pi(\langle 0\!:\!s_0, 1\!:\!s_1, \dots, m\!:\!s_m\rangle\mid 0\!:\!s_0)\\
\nonumber =\ & Pr_{Tr(D, m)}(\langle 0\!:\!s_0, 1\!:\!s_1, \dots, m\!:\!s_m\rangle \mid 0\!:\!s_0\wedge C_\pi)\cdot E[U_{Tr(D, m)}(\langle 0\!:\!s_0, 1\!:\!s_1, \dots, m\!:\!s_m\rangle\wedge C_\pi)]\\
\nonumber =\ &\text{(By Proposition \ref{prop:policy-to-action})}\\
\nonumber \ &Pr_{Tr(D, m)}(\langle 0\!:\!s_0, 1\!:\!s_1, \dots, m\!:\!s_m\rangle \mid 0\!:\!s_0\wedge 0\!:\!\pi(s_0, 0), \dots, m-1\!:\!\pi(s_{m-1}, m-1))\times\\
\nonumber &\ \ \ \ \ \ \ \ \ \ E[U_{Tr(D, m)\cup {C_\pi}}(\langle 0\!:\!s_0, 1\!:\!s_1, \dots, m\!:\!s_m\rangle\wedge C_\pi)]\\
\nonumber =\ &\text{(By Corollary 1 in \cite{lee18aprobabilistic})}\\
\nonumber \ &\underset{i\in \{0, \dots, m-1\}}{\prod}p(\langle s_i, \pi(s_i, i), s_{i+1}\rangle)\cdot E[U_{Tr(D, m)\cup {C_\pi}}(\langle 0\!:\!s_0, 1\!:\!s_1, \dots, m\!:\!s_m\rangle\wedge C_\pi)]\\
\nonumber =\ & \text{(By Proposition \ref{prop:step-wise-utility})}\\
\nonumber \ & \underset{i\in \{0, \dots, m-1\}}{\prod}p(\langle s_i, \pi(s_i, i), s_{i+1}\rangle)\cdot \underset{i\in \{0, \dots, m-1\}}{\sum} u(s_i, \pi(s_i, i), s_{i+1})\\
\nonumber =\ & \underset{i\in \{0, \dots, m-1\}}{\prod}P(s_i, \pi(s_i, i), s_{i+1})\cdot \underset{i\in \{0, \dots, m-1\}}{\sum} R(s_i, \pi(s_i, i), s_{i+1})\\
\nonumber =\ & R_{M(D)}(\langle s_0, \pi(s_0, )\dots, \pi(s_{m-1}), s_m\rangle)\cdot P_{M(D)}(\langle s_0, \pi(s_0, )\dots, \pi(s_{m-1}), s_m\rangle).
\end{align}
\qed
\end{proof}
\EOCC

It follows that the policy optimization problem for $\pbcp$ action descriptions coincides with the policy optimization problem for MDP with finite horizon.


\begin{thm}\label{thm:pBC-plus-to-mdp}
For any nonnegative integer $m$ and an initial state $s_0\in {\bf S}$ that is consistent with $D_{init}$, we have
\[
\underset{\text{$\pi$ is a non-stationary policy}}{\rm argmax}\ E[U_{Tr(D, m)}(C_{\pi,  m}\wedge \langle s_0\rangle^t)]
=\underset{\text{$\pi$ is a non-stationary policy}}{\rm argmax}\ \i{ER}_{M(D)}(\pi, s_0).
\]
\end{thm}

\BOCCC
\noindent
{\bf Example~\ref{eg:simple} continued}\ \ \ 
{\sl 
Figure \ref{fig:simple-transition-system} illustrates the MDP $M(D^{simple})$ for the simple action description $D^{simple}$ in Example \ref{eg:simple}.
\begin{figure}
    \centering
    \includegraphics[width=0.8\textwidth]{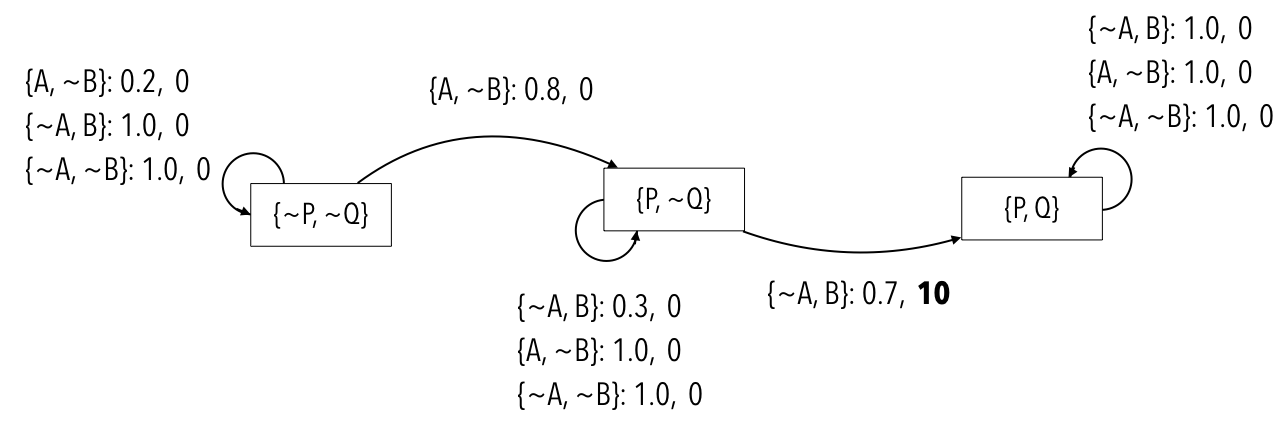}
    \caption{MDP $M(D^{simple})$}
    \label{fig:simple-transition-system}
\end{figure}
\BOCC
is the MDP $\langle S, A, P, R\rangle$ where
\begin{itemize}
\item $S=\{\{\sneg P, \sneg Q\}, \{P, \sneg Q\}, \{P, Q\}\}$;
\item $A = \{\{\sneg A, \sneg B\}, \{\sneg A, B\}\}, \{A, \sneg B\}\}$;
\item $P$ is defined as follows:
\begin{align}
\nonumber P(s, \{\sneg A, \sneg B\}, s)& =1 \text{ for all $s\in S$},\\
\nonumber P(\{\sneg P, \sneg Q\}, \{A, \sneg B\}, \{\sneg P, \sneg Q\})& =0.2,\\
\nonumber P(\{\sneg P, \sneg Q\}, \{A, \sneg B\}, \{P, \sneg Q\})& =0.8,\\
\nonumber P(\{\sneg P, \sneg Q\}, \{\sneg A, B\}, \{\sneg P, \sneg Q\})& =1,\\
\nonumber P(\{P, \sneg Q\}, \{A, \sneg B\}, \{P, \sneg Q\})& =1,\\
\nonumber P(\{P, \sneg Q\}, \{\sneg A, B\}, \{P, Q\})& =0.7,\\
\nonumber P(\{P, \sneg Q\}, \{\sneg A, B\}, \{P, \sneg Q\})& =0.3.
\end{align}
For $s, a, s'$ combination not mentioned above, $P(s, a, s') = 0$;
\item $R$ is defined as
\[
R(s, a, s') = \begin{cases}
10 & \text{if $s' = \{P, Q\}$},\\
0 & \text{otherwise}.
\end{cases}
\]
\end{itemize}
\EOCC
\BOCC
Consider one particular state sequence $\langle \{\sim P, \sim Q\}, \{P, \sim Q\}, \{P, Q\}\rangle$ and the policy $\pi$ defined as $$\pi(s) = \begin{cases}
A & \text{if $s=\{\sim P, \sim Q\}$}\\
B & \text{if $s=\{P, \sim Q\}$}\\
\emptyset & \text{otherwise.}
\end{cases}$$
We have
\begin{align}
\nonumber &U^{D^{simple}}_\pi(\langle \{\sim P, \sim Q\}, \{P, \sim Q\}, \{P, Q\}\rangle)\\
\nonumber  =\ & P_{Tr(D^{simple}, 2)}(\langle \{\sim P, \sim Q\}, \{P, \sim Q\}, \{P, Q\}\rangle \mid  0:\sim P\wedge0:\sim Q\wedge C_{\pi, 2})\cdot \\
\nonumber &\ \ \ \ \ \ E[U_{Tr(D, m)}(\langle \{\sim P, \sim Q\}, \{P, \sim Q\}, \{P, Q\}\rangle\wedge C_{\pi, 2})]\\
\nonumber  =\ & \frac{0.5 \times 0.5\times 0.8\times 0.7}{0.5 \times 0.5}\times 10\\
\nonumber =\ & 5.6
\end{align}
On the other hand,
\begin{align}
\nonumber & R_{M(D^{simple})}(\langle \{\sim P, \sim Q\},\pi(\{\sim P, \sim Q\}), \{P, \sim Q\}, \pi(\{P, \sim Q\}), \{P, Q\}\rangle)\cdot \\
\nonumber &\ \ \ \ \ \ \ \ P_{M(D^{simple})}(\langle \{\sim P, \sim Q\},\pi(\{\sim P, \sim Q\}), \{P, \sim Q\}, \pi(\{P, \sim Q\}), \{P, Q\}\rangle)\\
=\ & R_{M(D^{simple})}(\langle \{\sim P, \sim Q\},\{A, \sim B\}, \{P, \sim Q\}, \{\sim A, B\}, \{P, Q\}\rangle)\cdot \\
\nonumber &\ \ \ \ \ \ \ \ P_{M(D^{simple})}(\langle \{\sim P, \sim Q\},\{A, \sim B\}, \{P, \sim Q\}, \{\sim A, B\}, \{P, Q\}\rangle)\\
\nonumber =\ & 10 \times (0.8 \times 0.7)\\
\nonumber =\ & 5.6.
\end{align}
\EOCC
}
\EOCCC



Theorem~\ref{thm:pBC-plus-to-mdp} justifies using an implementation of $\dtlpmln$ to compute optimal policies of MDP $M(D)$ as well as using an MDP solver to compute optimal policies of the $\pbcp$ descriptions. 
Furthermore, the theorems above allow us to check the properties of MDP $M(D)$ by using formal properties of $\lpmln$, such as whether a certain state is reachable in a given number of steps.

\section{$\pbcp$ as a High-Level Representation Language of MDP}\label{sec:block-world}

An action description consists of causal laws in a human-readable form describing the action domain in a compact and high-level way, whereas it is non-trivial to describe an MDP instance directly from the domain description in English. The result in the previous section shows how to construct an MDP instance $M(D)$ for a $\pbcp$ action description $D$ so that the solution to the policy optimization problem of $D$ coincides with that of MDP $M(D)$. In that sense, $\pbcp$ can be viewed as a high-level representation language for MDP.

Since $\lpmln$ programs are weighted rules under the stable model semantics, and the semantics of $\pbcp$ is defined in terms of $\lpmln$,  $\pbcp$ inherits the nonmonotonicity of the stable model semantics to be able to compactly represent recursive definitions or transitive closure \cite{erdem16applications}. The static laws in $\pbcp$ prune out invalid states to ensure that only meaningful value combinations of fluents will be given to MDP as states, thus reducing the size of state space at the MDP level. To demonstrate this, we show how Example \ref{eg:block} (Robot and Blocks) can be represented in $\pbcp$ as follows.


First we define the signature of the action description: $x, x_1, x_2$ are schematic variables\footnote{An expression with schematic variables is a shorthand for the set of expressions obtained by replacing every variable in the original expression with every term in the domain of the variable.} that range over ${\tt B1}$, ${\tt B2}$, ${\tt B3}$; $r, r_1, r_2$ range over ${\tt R1}$, ${\tt R2}$. $\i{TopClear}(x)$, $\i{Above}(x_1, x_2)$, and 
$\i{GoalNotAchieved}$ are Boolean statically determined fluent constants; $\i{In}(x)$ is a regular fluent constant with domain $\{{\tt R1},{\tt R2}\}$, and $\i{OnTopOf}(x_1,x_2)$ is a Boolean regular fluent constant. $\i{MoveTo}(x,r)$ and $\i{StackOn}(x_1,x_2)$ are action constants and $\i{Pf\_Move}$ is a Boolean pf constant. In this example, we make the goal state absorbing, i.e., when all the blocks are already in {\tt R2}, then all actions have no effect. 

\BOC
\NB{This can be shrinken}
\vspace{0.2cm}
\hrule
\begin{tabbing}
Notation:  $x, x_1, x_2$ range over ${\tt B1}$, ${\tt B2}$, ${\tt B3}$; $l, l1$ ranges over ${\tt L1}$, ${\tt L2}$ \\
Statically determined fluent constant:         \hskip 2cm  \=Domains:\\
$\;\;\;$ $\i{TopClear}(x)$                 \>$\;\;\;$ Boolean\\ 
$\;\;\;$ $\i{Above}(x_1, x_2)$                 \>$\;\;\;$ Boolean\\ 
$\;\;\;$ $\i{GoalNotAchieved}$                 \>$\;\;\;$ Boolean\\ 
Regular fluent constants:          \hskip 4cm  \=Domains:\\
$\;\;\;$ $\i{At}(x)$                 \>$\;\;\;$ $\{{\tt L1}, {\tt L2}\}$\\ 
$\;\;\;$ $\i{OnTopOf}(x_1, x_2)$                 \>$\;\;\;$ Boolean\\ 
Action constants:                          \>Domains:\\
$\;\;\;$ $\i{MoveTo(x, l)}$  \>$\;\;\;$ Boolean\\
$\;\;\;$ $\i{StackOn}(x_1, x_2)$  \>$\;\;\;$ Boolean\\ 
Pf constants:                          \>Domains:\\
$\;\;\;$ $\i{\i{Pf\_Move}}$                    \>$\;\;\;$ Boolean 
\end{tabbing}
\hrule
\vspace{0.2cm}
\EOC

\BOCC
{\tiny
\begin{table}[h!]
\begin{tabular}{|l|p{8cm}|}
\hline
\textbf{causal laws} & \textbf{Explanations} \\ \hline
$\begin{array}{l}
\i{MoveTo}(x, l)\ \causes\ \i{At}(x) = l\ \iif\ \i{Pf\_Move}\\
\caused\ \i{Pf\_Move}=\{\true: p, \false: 1-p\}
\end{array}$&
Moving block $x$ to location $l$ causes $x$ to be at $l$ with probability $p$
\\ \hline
$\begin{array}{l}
\i{MoveTo}(x_1, l_2)\ \causes\ \sim\!\i{OnTopOf}(x_1, x_2)\ \\
\ \ \ \iif\ \i{Pf\_Move}\wedge \i{At}(x_1) = l_1\wedge \i{OnTopOf}(x_1, x_2)\wedge l_1\neq l_2
\end{array}$ &
Successfully Moving a block $x_1$ to a location $l_2$ causes $x_1$ to be no longer underneath the block $x_2$ that $x_1$ was underneath in the previous step, if $l_2$ is different from where $x_2$ is.
\\ \hline
$\begin{array}{l}
\i{StackOn}(x_1, x_2)\ \causes\ \i{OnTopOf}(x_1, x_2)\ \\
\ \ \ \iif\ x_1\neq x_2 \wedge \i{TopClear}(x_2)\wedge \i{At}(x_1) = l\wedge \i{At}(x_2) = l
\end{array}$&
Stacking a block $x_1$ on another block $x_2$ causes $x_1$ to be on top of $x_2$, if the top of $x_2$ is clear, and $x_1$ and $x_2$ are at the same location.\\ \hline
$\begin{array}{l}
\i{StackOn}(x_1, x)\ \causes\ \sim\!\i{OnTopOf}(x_1, x_2)\ \iif\\
\ \ \ x_1\neq x_2 \wedge \i{TopClear}(x_2)\wedge \i{At}(x_1) = l\wedge \i{At}(x_2) = l\wedge \\
\ \ \ \i{OnTopOf}(x_1, x_2) \wedge x_2\neq x
\end{array}$&
Stacking a block $x_1$ on another block $x$ causes $x_1$ to be no longer on top of the block $x_1$ where $x_1$ was originally on top of.\\ \hline
$\begin{array}{l}
\constraint\ \neg(\i{OnTopOf}(x_1, x) \wedge \i{OnTopOf}(x_2, x)\wedge x_1 \neq x_2)
\end{array}$&
Two different blocks cannot be on top of the same block.\\ \hline
$\begin{array}{l}
\constraint\ \neg(\i{OnTopOf}(x, x_1) \wedge \i{OnTopOf}(x, x_2)\wedge x_1 \neq x_2)
\end{array}$&
One block cannot be on top of two different blocks.\\  \hline
$\begin{array}{l}
\default\ \i{TopClear}(x)\\
\caused\ \sim\!\i{TopClear}(x)\ \iif\ \i{OnTopOf}(x_1, x)
\end{array}$&
By default, the top of a block $x$ is clear. It is not clear if there is another block $x_1$ that is on top of it.\\ \hline
$\begin{array}{l}
\caused\ \i{Above}(x_1, x_2)\ \iif\ \i{OnTopOf}(x_1, x_2)\\
\caused\ \i{Above}(x_1, x_2)\ \iif\ \i{Above}(x_1, x)\wedge \i{Above}(x, x_2) 
\end{array}$&
The relation $Above$ between two blocks is the transitive closure of the relation $\i{OnTopOf}$: A block $x_1$ is above another block $x_2$ if $x_1$ is on top of $x_2$, or there is another block $x$ such that $x_1$ is above $x$ and $x$ is above $x_2$.\\ \hline
$\begin{array}{l}
\caused\ \bot\ \iif\ \i{Above}(x_1, x_2)\wedge\i{Above}(x_2, x_1)
\end{array}$&
One block cannot be above itself; Two blocks cannot be above each other.\\ \hline
$\begin{array}{l}
\caused\ \i{At}(x_1) = l\ \iif\ \i{Above}(x_1, x_2)\wedge \i{At}(x_2) = l_2
\end{array}$&
If a block $x_1$ is above another block $x_2$, then $x_1$ has the same location as $x_2$.\\ \hline
$\begin{array}{l}
\reward\ -1\ \iif\  \top\ \after\ \i{MoveTo}(x, l)
\end{array}$&
Each moving action has a cost of $1$. \\ \hline
$\begin{array}{l}
\reward\ 10\ \iif\  \sim\!\i{GoalNotAchieved}\ \after\ \i{GoalNotAchieved}
\end{array}$&
Achieving the goal when the goal is not previously achieved yields a reward of $10$\\ \hline
$\begin{array}{l}
\caused\ \i{GoalNotAchieved}\ \iif\ \i{At}(x) = l\wedge l\neq {\tt L2}\\
\caused\ \sim\!\i{GoalNotAchieved}\ \iif\ \neg \i{GoalNotAchieved}
\end{array}$&
The goal is not achieved if there exists a block $x$ that is not at ${\tt L2}$. It is achieved otherwise.\\ \hline
$\begin{array}{l}
\inertial \ \i{At}(x) = l, \i{OnTopOf}(x_1,x_2)
\end{array}$&
The fluent $\i{At}(x, l)$ observes commonsense law of inertia. \\ \hline
$\begin{array}{l}
a_1\wedge a_2\ \causes\ \bot\\
\text{for each distinct pair of ground action constants $a_1$ and $a_2$}
\end{array}$&
At most one action can occur each time step\\ \hline
\end{tabular}
\caption{$\pbcp$ Action Description for Block Example}
\label{tab:block-example}
\end{table}
}

Table \ref{tab:block-example} shows the causal laws in the $\pbcp$ action description for this example.

\EOCC

Moving block $x$ to room $r$ causes $x$ to be in $r$ with probability $p$:
\[
\ba l
\i{MoveTo}(x, r)\ \causes\ \i{In}(x) = r\ \iif\ \i{Pf\_Move}\wedge \i{GoalNotAchieved}\\
\caused\ \i{Pf\_Move}=\{\true: p, \false: 1-p\} . 
\ea 
\] 
If a block $x_1$ is on top of another block $x_2$, then successfully moving $x_1$ to a different room $r_2$ causes $x_1$ to be no longer on top of $x_2$:
\[ 
\ba l
\i{MoveTo}(x_1, r_2)\ \causes\ \sim\!\i{OnTopOf}(x_1, x_2)\ \\
\hspace{1.5cm}\iif\ \i{Pf\_Move}\wedge \i{In}(x_1) = r_1\wedge \i{OnTopOf}(x_1, x_2)\wedge \i{GoalNotAchieved}\ \ \ \  (r_1\neq r_2). 
\ea 
\]
Stacking a block $x_1$ on another block $x_2$ causes $x_1$ to be on top of $x_2$, if the top of $x_2$ is clear, and $x_1$ and $x_2$ are at the same location:
\begin{align}
\nonumber & \i{StackOn}(x_1, x_2)\ \causes\ \i{OnTopOf}(x_1, x_2)\\
\nonumber & \hspace{1cm} \iif\ \i{TopClear}(x_2)\wedge \i{At}(x_1) = r\wedge 
   \i{At}(x_2) = r\wedge \i{GoalNotAchieved} 
\hspace{1cm}  (x_1\neq x_2). 
\end{align}
Stacking a block $x_1$ on another block $x_2$ causes $x_1$ to be no longer on top of the block $x$ where $x_1$ was originally on top of:
\begin{align}
\nonumber &\i{StackOn}(x_1, x_2)\ \causes\ \sim\!\i{OnTopOf}(x_1, x)\ \iif\   \i{TopClear}(x_2)\wedge \i{At}(x_1) = r\wedge \i{At}(x_2) = r\wedge\\
\nonumber &\ \ \ \ \ \ \ \ \ \ \ \ \i{OnTopOf}(x_1, x)\wedge \i{GoalNotAchieved} \ 
\hspace{3.5cm} (x_2\neq x, x_1\neq x_2). 
\end{align}
Two different blocks cannot be on top of the same block, and 
a block cannot be on top of two different blocks:
\[
\ba l
\constraint\ \neg(\i{OnTopOf}(x_1, x) \wedge \i{OnTopOf}(x_2, x))\ \hspace{2.5cm} (x_1 \neq x_2)  \\
\constraint\ \neg(\i{OnTopOf}(x, x_1) \wedge \i{OnTopOf}(x, x_2)) \ \hspace{2.5cm} (x_1 \neq x_2).
\ea
\]
By default, the top of a block $x$ is clear. It is not clear if there is another block $x_1$ that is on top of it:
\[
\ba l
\default\ \i{TopClear}(x)\\
\caused\ \sim\!\i{TopClear}(x)\ \iif\ \i{OnTopOf}(x_1, x).
\ea 
\]
The relation $Above$ between two blocks is the transitive closure of the relation $\i{OnTopOf}$: A block $x_1$ is above another block $x_2$ if $x_1$ is on top of $x_2$, or there is another block $x$ such that $x_1$ is above $x$ and $x$ is above $x_2$:
\[ 
\ba l
\caused\ \i{Above}(x_1, x_2)\ \iif\ \i{OnTopOf}(x_1, x_2)\\
\caused\ \i{Above}(x_1, x_2)\ \iif\ \i{Above}(x_1, x)\wedge \i{Above}(x, x_2) .
\ea 
\]
One block cannot be above itself; two blocks cannot be above each other:
\[
\ba l
\caused\ \bot\ \iif\ \i{Above}(x_1, x_2)\wedge\i{Above}(x_2, x_1).
\ea 
\]
If a block $x_1$ is above another block $x_2$, then $x_1$ has the same location as $x_2$:
\beq
\ba l
\caused\ \i{At}(x_1) = r\ \iif\ \i{Above}(x_1, x_2)\wedge \i{At}(x_2) = r.
\ea
\eeq{at-static}
Each moving action has a cost of $1$:
\[
\ba l
\reward\ -1\ \iif\  \top\ \after\ \i{MoveTo}(x, r).
\ea 
\]
Achieving the goal when the goal is not previously achieved yields a reward of $10$:
\[ 
\ba l
\reward\ 10\ \iif\  \sim\!\i{GoalNotAchieved}\ \after\ \i{GoalNotAchieved}.
\ea 
\]
The goal is not achieved if there exists a block $x$ that is not at ${\tt R2}$. It is achieved otherwise:
\[
\ba l
\caused\ \i{GoalNotAchieved}\ \iif\ \i{At}(x) = r\ \ ( r\neq {\tt R2})\\
\default\ \sim\!\i{GoalNotAchieved}. 
\ea
\]
$\i{At}(x)$ and $\i{OnTopOf}(x_1,x_2)$ are inertial:
\[
\ba l
\inertial \ \i{At}(x), \i{OnTopOf}(x_1,x_2).
\ea
\]
Finally, we add 
\ \ $ a_1\wedge a_2\ \causes\ \bot $ \ \ 
for each distinct pair of ground action constants $a_1$ and $a_2$, to ensure that at most one action can occur each time step.

It can be seen that stacking all blocks together and moving them at once would be the best strategy to move them to ${\tt R2}$. 

In the robot and blocks example, many value combinations of fluents do not lead to a valid state, such as
$$\{\i{OnTopOf}({\tt B1}, {\tt B2}), \i{OnTopOf}({\tt B2}, {\tt B1}), ...\},$$
where the two blocks ${\tt B1}$ and ${\tt B2}$ are on top of each other. Moreover, the fluents $\i{TopClear}(x)$ and $\i{Above}(x_1, x_2)$ are completely dependent on the value of the other fluents. There would be $2^{3+3\times 3 + 3 + 3\times 3}=2^{24}$ states if we define a state as any value combination of fluents. On the other hand, the static laws in the above action description reduce the number of states to only $(13 + 9)\times 2=44$. To see this, consider all possible configuration with $3$ blocks and $2$ locations. As illustrated in Figure \ref{fig:3-blocks}, there are $13$ possible configurations with $3$ blocks on the same side, and $9$ possible configurations with one block one one side and two on the other side. Each configureration can be mirrored to yield another configuration, so we have $(13 + 9)\times 2 = 44$ in total. This is aligned with the number of MDP states obtained from the $\pbcp$ action description according to Definition \ref{def:pbcp-mdp}.

\begin{figure}
    \centering
    \includegraphics[width=0.9\textwidth]{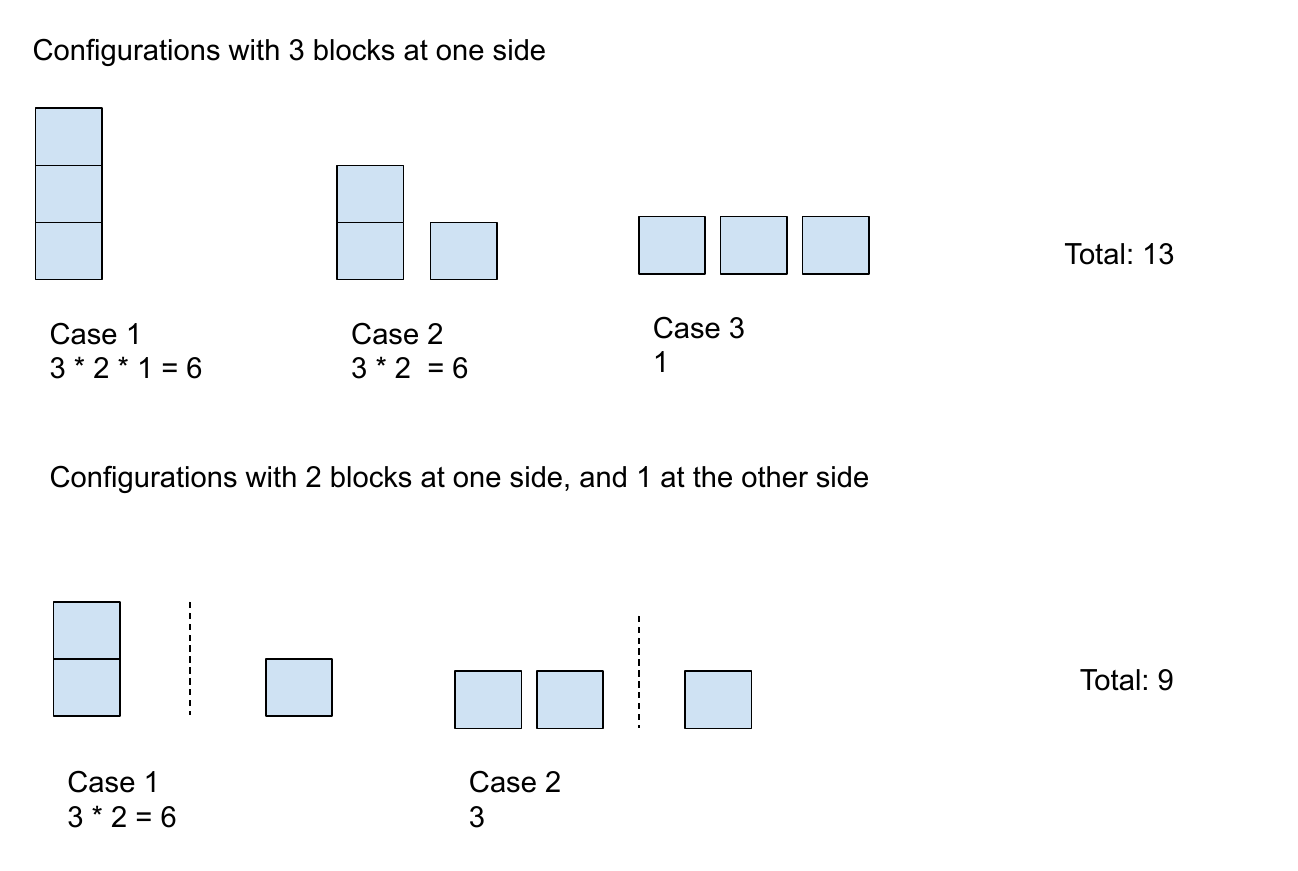}
    \caption{Possible Configurations with $3$ blocks and $2$ locations}
    \label{fig:3-blocks}
\end{figure}

Furthermore, in this example, $\i{Above}(x, y)$ needs to be defined as a transitive closure of $\i{OnTopOf}(x, y)$, so that the effects of $\i{StackOn}(x_1, x_2)$ can be defined in terms of the (inferred) spatial relation of blocks. Also, the static law~\eqref{at-static}
defines an indirect effect of $\i{MoveTo}(x, r)$. 


\BOCC
We implemented the prototype system {\sc pbcplus2mdp}, which takes an action description $D$ and time horizon $m$ as input, and finds an optimal policy by constructing the corresponding MDP $M(D)$ and invoking an MDP solver {\sc mdptoolbox}.\footnote{\url{https://pymdptoolbox.readthedocs.io}} 
The current system uses $\lpmln$ 1.0 
\cite{lee17computing} (\url{http://reasoning.eas.asu.edu/lpmln}) for exact inference to find states, actions, transition probabilities, and transition rewards. The system is publicly available at \url{https://github.com/ywang485/pbcplus2mdp}\cite{yi_wang_2020_3726430}, along with several examples.
The current system is not quite scalable because generating exact transition probability and reward matrices requires enumerating all stable models of $D_0$ and~$D_1$. 
\EOCC


\BOCC
\begin{algorithm}[h!]
{\footnotesize
\noindent {\bf Input: }
\begin{enumerate}
\item $Tr(D, m)$: A $\pbcp$ action description translated into $\lpmln$ program, parameterized with maxstep $m$, with states set ${\bf S}$ and action sets ${\bf A}$
\item $T$: time horizon
\item $\gamma$: discount factor
\end{enumerate}
\noindent {\bf Output: } Optimal policy

\noindent {\bf Procedure:}
\begin{enumerate}
\item Execute {\sc lpmln2asp} on $Tr(D, m)$ with $m=0$ to obtain all stable models of $Tr(D, 0)$; project each stable model of $Tr(D, 0)$ to only predicates corresponding to fluent constant (marked by {\tt fl\_} prefix); assign a unique number $idx(s)\in\{0, \dots, |{\bf S}|-1\}$ to each of the projected stable model $s$ of $Tr(D, 0)$;
\item Execute {\sc lpmln2asp} on $Tr(D, m)$ with $m=1$ and the clingo option {\tt --project} to project stable models to only predicates corresponding to action constant (marked by {\tt act\_} prefix); assign a unique number $idx(a)\in\{0, \dots, |{\bf A}|-1\}$ to each of the projected stable model $a$ of $Tr(D, 1)$;
\item Initialize 3-dimensional matrix $P$ of shape $(|{\bf A}|, |{\bf S}|, |{\bf S}|)$;
\item Initialize 3-dimensional matrix $R$ of shape $(|{\bf A}|, |{\bf S}|, |{\bf S}|)$;
\item For each state $s\in {\bf S}$ and action $a\in {\bf A}$:
\begin{enumerate}
\item execute {\sc lpmln2asp} on $Tr(D, m)\cup \{0:s\}\cup \{0:a\}\cup ST\_DEF$ with $m=1$ and the option {\tt -q "end\_state"}, where $ST\_DEF$ contains the rule
$$\{{\tt end\_state}(idx(s))\leftarrow 1:s\mid s\in {\bf S}\}.$$
\item Obtain $P_{Tr(D, 1)}(1: s' \mid 0:s, 0:a)$ by extracting the probability of $P_{Tr(D, 1)}({\tt end\_state}(idx(s')) \mid 0:s, 0:a)$ from the output;
\item $P(idx(a), idx(s), idx(s'))\leftarrow P_{Tr(D, 1)}(1: s' \mid 0:s, 0:a)$;
\item Obtain $E[U_{Tr(D, 1)}(1: s', 0:s, 0:a)]$ from the output by selecting an arbitrary answer set returned that satisfies $1: s'\wedge 0:s\wedge 0:a$ and sum up the first arguments of all predicates named {\tt utility} (By Proposition \ref{prop:history-determines-utility}, this is equivalent to $E[U_{Tr(D, 1)}(1: s', 0:s, 0:a)]$). 
\item $R(idx(a), idx(s), idx(s'))\leftarrow E[U_{Tr(D, 1)}(1: s', 0:s, 0:a)]$;
\end{enumerate}
\item Call finite horizon policy optimization algorithm of {\sc pymdptoolbox} with transition matrix $P$, reward matrix $R$, time horizon $T$ and discount factor $\gamma$; return the output.

\end{enumerate}
\caption{{\sc pbcplus2mdp} system}
\label{alg:pbcplus2mdp}
}
\end{algorithm}
\EOCC
\BOCC
We measure the scalability of our system {\sc pbcplus2mdp} on the robot and blocks example Table \ref{tab:pbcplus2mdp-system-analysis} shows the running statistics of finding the optimal policy for different number of blocks. For all of the running instances, maximum time horizon is set to be $10$, as in all of the instances, the smallest number of steps in a shortest possible action sequence achieving the goal is less than $10$. The experiments are performed on a machine with 4 Intel(R) Core(TM) i5-2400 CPU with OS Ubuntu 14.04.5 LTS and 8 GB memory.

\begin{table}[htb]
\centering
\vspace{-0.5cm}
\begin{tabular}{|c|c|c|c|c|c|}
\hline
{\bf \# Blocks} & {\bf \# State} & {\bf \# Actions} & {\bf $\lpmln$ Solving Time} & {\bf MDP Solving Time} & {\bf Overall Solving Time}\\
\hline
1 & 2 & 4 & 0.902s & 0.0005 & 1.295s\\
2 & 8 & 9 & 0.958s & 0.0014s & 1.506s\\
3 & 44 & 16 & 1.634s & 0.0017s & 2.990s\\
4 & 304 & 25 & 12.256s & 0.0347s & 27.634s\\
5 & 2512 & 36 & 182.190s & 2.502 & 10m23.929s\\
6 & 24064 & 49 & $>$ 1 hr & - & - \\
\hline
\end{tabular}
\caption{Running Statistics of {\sc pbcplus2mdp} system}
\label{tab:pbcplus2mdp-system-analysis}
\end{table}

As can be seen from the table, the running time increases exponentially as the number of blocks increases. This is not surprising since the size of the search space also increases exponentially as the number of blocks increases. The bottleneck is the $\lpmln$ inference system, as it needs to enumerate every stable model to generate the set of states, the set of actions, and transition probabilities and rewards. The time spent on MDP solving is negligible.

[[
{\cblu Need to explain the diffrence in overall solving time: python... }
]]

System {\sc pbcplus2mdp} supports planning with infinite horizon. However, it should be noted that the semantics of an action description with infinite time horizon in terms of $\dtlpmln$ is not yet well established. In this case, the action description is only viewed as a high-level representation of an MDP.
\EOCC

\section{System {\sc pbcplus2mdp}}
\label{sec:system-pbcplus2mdp}

We implement system {\sc pbcplus2mdp}, which takes the $\lpmln$ translation of an action description $D$ and time horizon $m$ as input, and finds the optimal policy by constructing the corresponding MDP $M(D)$ and utilizing MDP policy optimization algorithms as black box. We use {\sc mdptoolbox}\footnote{\url{https://pymdptoolbox.readthedocs.io}} as our underlying MDP solver. The current system uses $\lpmln$ 1.0 ( \url{http://reasoning.eas.asu.edu/lpmln/index.html}) for exact inference to find states, actions, transition probabilities and transition rewards. The system is publically available at \url{https://github.com/ywang485/pbcplus2mdp} \cite{yi_wang_2020_3726430}, along with several examples.

We measure the scalability of our system {\sc pbcplus2mdp} on the robot and blocks example. Figure \ref{fig:pbcplus2mdp-system-analysis-2} shows the running statistics of finding the optimal policy for different number of blocks. For all of the running instances, maximum time horizon is set to be $10$, as in all of the instances, the smallest number of steps in a shortest possible action sequence achieving the goal is less than $10$. The discount factor is set to be $0.9$. The experiments are performed on a machine with 4 Intel(R) Core(TM) i5-2400 CPU with OS Ubuntu 14.04.5 LTS and 8 GB memory.

\BOCC
\begin{table}[htb]
\centering
\begin{tabular}{|c|c|c|c|c|c|}
\hline
{\bf \# Blocks} & {\bf \# State} & {\bf \# Actions} & {\bf MDP Generation Time} & {\bf MDP Solving Time} & {\bf Overall Solving Time}\\
\hline
1 & 2 & 4 & 0.902s & 0.0005 & 1.295s\\
2 & 8 & 9 & 0.958s & 0.0014s & 1.506s\\
3 & 44 & 16 & 1.634s & 0.0017s & 2.990s\\
4 & 304 & 25 & 12.256s & 0.0347s & 27.634s\\
5 & 2512 & 36 & 182.190s & 2.502 & 10m23.929s\\
6 & 24064 & 49 & $>$ 1 hr & - & - \\
\hline
\end{tabular}
\caption{Running Statistics of {\sc pbcplus2mdp} system}
\label{tab:pbcplus2mdp-system-analysis-2}
\end{table}
\EOCC

\begin{figure}
 \begin{center}
    \includegraphics[width=0.98\textwidth]{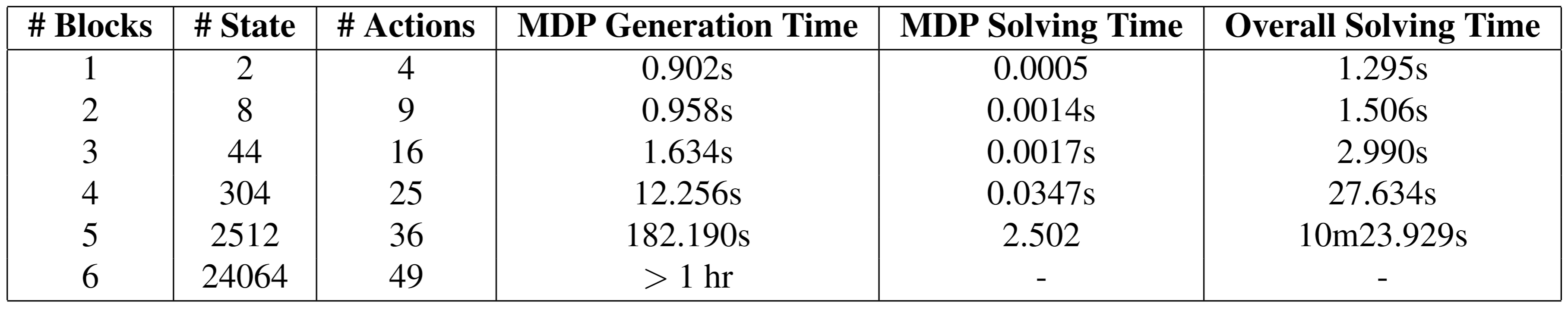}
    \end{center}
    \caption{Running Statistics of {\sc pbcplus2mdp} system}
    \label{fig:pbcplus2mdp-system-analysis-2}
\end{figure}

As can be seen from the table, the running time increases exponentially as the number of blocks increases. This is not surprising since the size of the search space increases exponentially as the number of blocks increases. The bottleneck is the $\lpmln$ inference system, as it needs to enumerate every stable model to generate the set of states, the set of actions, transition probabilities and rewards. Time spent on MDP planning is negligible.

Since the bottleneck is the size of the domain, one potential direction of improving the scalability would be to represent the action description at first-order level, and then utilize planning algorithms for first-order MDPs, such as \cite{boutilier01symbolic,yoon02inductive,wang08first,sanner09practical}, to compute domain-independent policies. This method requires a solver for first-order $\lpmln$. We leave this for future work.

System {\sc pbcplus2mdp} supports planning with infinite horizon. However, it should be noted that the semantics of an action description with infinite time horizon in terms of $DT-\lpmln$ is not yet well established. In this case, the action description is only viewed as a high-level representation of an MDP.

\BOCC
Algorithm \ref{alg:pbcplus2mdp} details how {\sc pbcplus2mdp} solves the policy optimization problem of an $\pbcp$ action description. 

\begin{algorithm}[h!]
{\footnotesize
\noindent {\bf Input: }
\begin{enumerate}
\item $D_m$: A $p\cal{BC}+$ action description translated into $\lpmln$ program, parameterized with maxstep $m$, with states set ${\bf S}$ and action sets ${\bf A}$
\item $T$: time horizon
\item $\gamma$: discount factor
\end{enumerate}
\noindent {\bf Output: } Optimal policy

\noindent {\bf Procedure:}
\begin{enumerate}
\item Execute system {\sc lpmln} on $D_0$ to obtain all stable models of $D_0$; project each stable model of $D_0$ to only predicates corresponding to fluent constant (marked by {\tt fl\_} prefix); assign a unique number $idx(s)\in\{0, \dots, |{\bf S}|-1\}$ to each of the projected stable model $s$ of $Tr(D, 0)$;
\item Execute {\sc lpmln2asp} on $Tr(D, m)$ with $m=1$ and the clingo option {\tt --project} to project stable models to only predicates corresponding to action constant (marked by {\tt act\_} prefix); assign a unique number $idx(a)\in\{0, \dots, |{\bf A}|-1\}$ to each of the projected stable model $a$ of $Tr(D, 1)$;
\item Initialize 3-dimensional matrix $P$ of shape $(|{\bf A}|, |{\bf S}|, |{\bf S}|)$;
\item Initialize 3-dimensional matrix $R$ of shape $(|{\bf A}|, |{\bf S}|, |{\bf S}|)$;
\item Construct ASP rules $ST\_ACT\_DEF$:
$$\{{\tt start\_state}(idx(s))\leftarrow 0:s\mid s\in {\bf S}\}\cup\{{\tt end\_state}(idx(s))\leftarrow 1:s\mid s\in {\bf S}\}\cup\{{\tt action\_idx}(idx(a))\leftarrow 0:a\mid a\in {\bf A}\}$$
\item Execute {\sc lpmln2asp} on $Tr(D, m)\cup ST\_ACT\_DEF$ with $m=1$ and the option {\tt -all} to obtain all stable models of $Tr(D, 1)$; Store all the stable models of $Tr(D, 1)$ with a dictionary $AS$ so that 
\begin{align}
\nonumber AS(s, a, s')= & \{(X, P_{Tr(D, 1)}(X))\mid \text{$X$ is the stable model of $Tr(D, 1)$ such that} \\
\nonumber & \text{${\tt start\_state}(idx(s))$, ${\tt end\_state}(idx(s'))$, ${\tt action\_idx}(idx(a))\in X$}
\end{align}
\item For each state $s\in {\bf S}$ and action $a\in {\bf A}$:
\begin{enumerate}
\item Obtain $P_{Tr(D, 1)}(1: s' \mid 0:s, 0:a)$ as $\frac{\underset{(X, p)\in AS(s, a, s')}{\sum} p}{\underset{\substack{s'\in {\bf S}\\(X, p)\in AS(s, a, s')}}{\sum} p}$;
\item Obtain $E[U_{Tr(D, 1)}(1: s', 0:s, 0:a)]$ by arbitrarily picking one $(X, p)$ from $AS(s, a, s')$ and summing up the first arguments of all predicates named {\tt utlity};
\item $P(idx(a), idx(s), idx(s'))\leftarrow P_{Tr(D, 1)}(1: s' \mid 0:s, 0:a)$;
\item $R(idx(a), idx(s), idx(s'))\leftarrow E[U_{Tr(D, 1)}(1: s', 0:s, 0:a)]$;
\end{enumerate}
\item Call finite horizon policy optimization algorithm of {\sc pymdptoolbox} with transition matrix $P$, reward matrix $R$, time horizon $T$ and discount factor $\gamma$; return the output.

\end{enumerate}
\caption{{\sc pbcplus2mdp} system}
\label{alg:pbcplus2mdp}
}
\end{algorithm}

\begin{example}
The $\lpmln$ translation of the $\pbcp$ action description in Example 2 is listed below:
\begin{lstlisting}
astep(0..m-1).
step(0..m).
boolean(t; f).

block(b1; b2; b3).
location(l1; l2).

%% UEC
:- fl_Above(X1, X2, t, I), fl_Above(X1, X2, f, I).
:- not fl_Above(X1, X2, t, I), not fl_Above(X1, X2, f, I), block(X1), block(X2), step(I).
:- fl_TopClear(X, t, I), fl_TopClear(X, f, I).
:- not fl_TopClear(X, t, I), not fl_TopClear(X, f, I), block(X), step(I).
:- fl_GoalNotAchieved(t, I), fl_GoalNotAchieved(f, I).
:- not fl_GoalNotAchieved(t, I), not fl_GoalNotAchieved(f, I), step(I).

:- fl_At(X, L1, I), fl_At(X, L2, I), L1 != L2.
:- not fl_At(X, l1, I), not fl_At(X, l2, I), block(X), step(I).
:- fl_OnTopOf(X1, X2, t, I), fl_OnTopOf(X1, X2, f, I).
:- not fl_OnTopOf(X1, X2, t, I), not fl_OnTopOf(X1, X2, f, I), block(X1), block(X2), step(I).

:- act_StackOn(X1, X2, t, I), act_StackOn(X1, X2, f, I).
:- not act_StackOn(X1, X2, t, I), not act_StackOn(X1, X2, f, I), block(X1), block(X2), astep(I).
:- act_MoveTo(X, L, t, I), act_MoveTo(X, L, f, I).
:- not act_MoveTo(X, L, t, I), not act_MoveTo(X, L, f, I), block(X), location(L),astep(I).

:- pf_Move(t, I), pf_Move(f, I).
:- not pf_Move(t, I), not pf_Move(f, I), astep(I).

% ---------- PF(D) ----------
%% Probability Distribution
@log(0.8) pf_Move(t, I) :- astep(I).
@log(0.2) pf_Move(f, I) :- astep(I).

%% Initial State and Actions are Random
{fl_OnTopOf(X1, X2, B, 0)} :- block(X1), block(X2), boolean(B).
{fl_At(X, L, 0)} :- block(X), location(L), boolean(B).
{act_StackOn(X1, X2, B, I)} :- block(X1), block(X2), boolean(B), astep(I).
{act_MoveTo(X, L, B, I)} :- block(X), location(L), boolean(B), astep(I).

%% No Concurrency
:- act_StackOn(X1, X2, t, I), act_StackOn(X3, X4, t, I), astep(I), X1 != X3.
:- act_StackOn(X1, X2, t, I), act_StackOn(X3, X4, t, I), astep(I), X2 != X4.
:- act_MoveTo(X1, L1, t, I), act_MoveTo(X2, L2, t, I), astep(I), X1 != X2.
:- act_MoveTo(X1, L1, t, I), act_MoveTo(X2, L2, t, I), astep(I), L1 != L2.
:- act_StackOn(X1, X2, t, I), act_MoveTo(X3, L, t, I), astep(I).

%% Static Laws
fl_GoalNotAchieved(t, I) :- fl_At(X, L, I), L != l2.
fl_GoalNotAchieved(f, I) :- not fl_GoalNotAchieved(t, I), step(I).
:- fl_OnTopOf(X1, X, t, I), fl_OnTopOf(X2, X, t, I), X1 != X2.
:- fl_OnTopOf(X, X1, t, I), fl_OnTopOf(X, X2, t, I), X1 != X2.
fl_Above(X1, X2, t, I) :- fl_OnTopOf(X1, X2, t, I).
fl_Above(X1, X2, t, I) :- fl_Above(X1, X, t, I), fl_Above(X, X2, t, I).
:- fl_Above(X1, X2, t, I), fl_Above(X2, X1, t, I).
fl_At(X1, L, I) :- fl_Above(X1, X2, t, I), fl_At(X2, L, I).
fl_Above(X1, X2, f, I) :- not fl_Above(X1, X2, t, I), block(X1), block(X2), step(I).
fl_TopClear(X, f, I) :- fl_OnTopOf(X1, X, t, I).
fl_TopClear(X, t, I) :- not fl_TopClear(X, f, I), block(X), step(I).

%% Fluent Dynamic Laws
fl_At(X, L, I+1) :- act_MoveTo(X, L, t, I), pf_Move(t, I), fl_GoalNotAchieved(t, I).
fl_OnTopOf(X1, X2, t, I+1) :- act_StackOn(X1, X2, t, I), X1 != X2, fl_TopClear(X2, t, I), not fl_Above(X2, X1, t, I), fl_At(X1, L, I), fl_At(X2, L, I), fl_GoalNotAchieved(t, I).
fl_OnTopOf(X1, X2, f, I+1) :- act_MoveTo(X1, L2, t, I), pf_Move(t, I), fl_At(X1, L1, I), fl_OnTopOf(X1, X2, t, I), L1 != L2, fl_GoalNotAchieved(t, I).
fl_OnTopOf(X1, X, f, I+1) :- act_StackOn(X1, X2, t, I), X1 != X2, fl_TopClear(X2, t, I), not fl_Above(X2, X1, t, I), fl_At(X1, L, I), fl_At(X2, L, I), fl_OnTopOf(X1, X, t, I), X != X2, fl_GoalNotAchieved(t, I).
{fl_OnTopOf(X1, X2, B, I+1)} :- fl_OnTopOf(X1, X2, B, I), astep(I), boolean(B).
{fl_At(X, L, I+1)} :- fl_At(X, L, I), astep(I), boolean(B).

%% Utility Laws
utility(-1, X, L, I) :- act_MoveTo(X, L, t, I).
utility(10) :- fl_GoalNotAchieved(f, I+1), fl_GoalNotAchieved(t, I).
\end{lstlisting}
\end{example}
\EOCC

\section{Related Work}\label{sec:related-work}

There have been quite a few studies and attempts in defining factored representations of (PO)MDP, with feature-based state descriptions and more compact, human-readable action definitions. PPDDL \cite{younes04ppddl1} extends PDDL with constructs for describing probabilistic effects of actions and reward from state transitions. 
\BOCC
A PPDDL action schema may contain
\[
({\tt probabilistic}\ p_1\ e_1\ \dots\ p_k\ e_k),
\]
which associates probabilities $p_1, \dots, p_k$ to effects $e_1, \dots, c_k$. 
The effects are arbitrary formulas that can again contain probabilistic effects. The arbitrary nesting of effects allows flexible representations of effects of actions. The transitions and rewards of the underlying MDP, though, is defined assuming that the action effects are turned into some normal form. The states of the MDP are all possible value combinations of the fluents, and the entries of transition and reward matrices are filled out according to the satisfaction of the effect w.r.t. the corresponding transition.  
\EOCC
One limitation of PPDDL is the lack of static causal laws, which prohibits PPDDL from expressing recursive definitions or transitive closure. This may yield a large state space to explore as discussed in Section~\ref{sec:block-world}.
\BOCC
In many domains, features used to define states are not independent from each other. In these domains, some combination of fluent values do not yield a valid state. Although invalid combinations of fluent values can be simulated by unreachable states, including all such invalid combinations as states unnecessarily increases the size of the domain. With {\sc pbcplus2mdp}, such invalid combinations can be pruned out with static laws in $\pbcp$, and removed from the state space given to the MDP solver. 

Another weakness of PPDDL lies in the limitation on expressivity due to simple satisfaction checking in determining valid transitions. Action effects involving recursive definitions or transitive closures are hard to be expressed, while in $\pbcp$, the stable model semantics makes it quite straightforward to express recursive definitions. 

Consider the Example \ref{eg:block}. This example is hard to represent in PPDDL because $\i{Above}(x, y)$ needs to be defined as a transitive closure of $\i{OnTopOf}(x, y)$. Also, the lack of state constraints from PPDDL would yield a unnecessarily large state space. 
\EOCC
RDDL (Relational Dynamic Influence Diagram Language) \cite{sanner10relational} improves the expressivity of PPDDL in modeling stochastic planning domains by allowing concurrent actions, continuous values of fluents, state constraints, etc. The semantics is defined in terms of lifted dynamic Bayes network extended with influence graph. A lifted planner can utilize the first-order representation and potentially achieve better performance. Still, indirect effects are hard to be represented in RDDL. Compared to PPDDL and RDDL, the advantages of $p\cal{BC}+$ are in its simplicity and expressivity originating from the stable model semantics, which allows for elegant representation of recursive definitions, defeasible behaviors, and indirect effects.


\BOCC
(\cite{baral02reasoning}) is yet another logic formalism for modeling MDP in an elaboration tolerant way. The syntax is similar to $\pbcp$. The semantics is more complex as it allows preconditions of actions and imposes less semantical
assumption. The concept of unknown variables associated with probability distributions is similar to pf constants in our setting. The language is more focused on representing probabilistic effects. Although it is intended to a compact representation of MDPs, we do not see much discussion on the relation between its semantics and MDP, in terms of solution to planning problems. We do not find a notion of reward or utility defined in the context of this language.
\EOCC

\cite{poole13framework} combines the Situation Calculus  \cite{mccarthy63situations} and the Independent Choice Logic (ICL) \cite{poole08independent} for probabilistic planning. The situation calculus is used to specify effects of actions, and ICL is used to model randomness in an action domain. The notion of choice alternatives in ICL is similar to pf constant in $\pbcp$, and an atomic choice resembles a value assignment to a pf constant. While $\pbcp$ allows uncertainty to be from both probabilities and logic, \cite{poole13framework} has the restriction that all the uncertainty comes from probabilities, i.e., the logic program is required to have a unique model once all the independent choices are fixed.  Another difference is that \cite{poole13framework} considers only acyclic logic program (which is not a restriction in $\pbcp$), and it is thus not straightforward to represent transitive closure such as ``a block being moved causes the block on top of it also being moved'', as in Example \ref{eg:block}. It is worth noting that \cite{poole13framework} allows a more compact and flexible representation of policies, where conditional expressions can be used to summarize the action to take for a set of states where a certain sensor value is observed.



\cite{Zhang15corpp} adopt ASP and P-Log \cite{baral09probabilistic} 
which respectively produces a refined set of states and a refined probability distribution over states that are then fed to POMDP solvers for low-level planning. The refined sets of states and probability distribution over states take into account commonsense knowledge about the domain, and thus improve the quality of a plan and reduce computation needed at the POMDP level. \cite{yang18peorl} adopts the (deterministic) action description language $\cal{BC}$ for high-level representations of the action domain, which defines high-level actions that can be treated as deterministic. 
Each action in the generated high-level plan is then mapped into more detailed low-level policies, which takes stochastic effects of low-level actions into account. 
{Similarly,
\cite{sridharan15reba} introduce a 
framework
with planning in a coarse-resolution transition model and a fine-resolution transition model. Action language ${\cal AL}_d$ is used for defining the two levels of transition models. The fine-resolution transition model is further turned into a POMDP for detailed planning with stochastic effects of actions and transition rewards. 
While a $p\cal{BC}+$ action description can fully capture all aspects of (PO)MDP including transition probabilities and rewards, the ${\cal AL}_d$ action description only provides states, actions and transitions with no quantitative information. 
\cite{leonetti16synthesis}, on the other hand, use symbolic reasoners such as ASP to reduce the search space for reinforcement learning based planning methods by generating partial policies from planning results generated by the symbolic reasoner. 
The exploration of the low-level RL module is constrained by actions that satisfy the partial policy.  
}

Another related work is \cite{ferreira17answer}, which combines ASP and reinforcement learning by using action language $\cal{BC}+$ as a meta-level description of MDP. The $\cal{BC}+$ action descriptions define non-stationary MDPs in the sense that the states and actions can change with new situations occurring in the environment. The algorithm ASP(RL) proposed in this work iteratively calls an ASP solver to obtain states and actions for the RL methods to learn transition probabilities and rewards, and updates the $\cal{BC}+$ action description with changes in the environment found by the RL methods, in this way finding optimal policy for a non-stationary MDP with the search space reduced by ASP. The work is similar to ours in that ASP-based high-level logical description is used to generate states and actions for MDP, but the difference is that we use an extension of ${\cal BC}$+ that expresses transition probabilities and rewards.

\section{Conclusion}
In this work, we bridge the gap between action language $\pbcp$ and Markov Decision Process by extending $\pbcp$ with the notion of utility, which allows $\pbcp$ to serve as an elaboration tolerant representation of MDP, as well as leveraging an MDP solver to compute a $\pbcp$ action description. Our main contributions are as follows.
\begin{itemize}
\item We extended $\lpmln$ with the notion of utility, resulting in $\dtlpmln$; we developed an approximate algorithm for maximizing expected utility in $\dtlpmln$;
\item Based on $\dtlpmln$, we extended $\pbcp$ with the notion of utility;
\item We showed that the semantics of $\pbcp$ can be alternatively defined in terms of Markov Decision Process;
\item We demonstrated how $\pbcp$ can serve as an elaboration tolerant representation of MDP;
\item We developed a prototype system {\sc pbcplus2mdp}, for finding optimal policies of $\pbcp$ action descriptions using an MDP solver.  
\end{itemize}

Formally relating action languages and MDP opens up interesting research to explore. Dynamic programming methods in MDP can be utilized to compute action languages. In turn, action languages may serve as a formal verification tool for MDP as well as a high-level representation language for MDP that describes an MDP instance in a succinct and elaboration tolerant way. As many reinforcement learning tasks use MDP as a modeling language, the work may be related to incorporating symbolic knowledge to reinforcement learning as evidenced by \cite{Zhang15corpp,yang18peorl,leonetti16synthesis}.

$\dtlpmln$ may deserve attention on its own for static domains. 
We expect that this extension of $\lpmln$ system that can handle utility can be a useful tool for verifying properties for MDP.





The theoretical results in this paper limit attention to MDP in the finite horizon case. When the maximum step $m$ is sufficiently large, we may view it as an approximation of the infinite horizon case, in which case, we allow discount factor $\gamma$ by replacing $v$ in \eqref{eq:utility-law-lpmln} with $\gamma^{i+1} v$. While it appears intuitive to extend the theoretical results in this paper to the infinite case, it requires extending the definition of $\lpmln$ to allow infinitely many rules, which we leave for future work. 


\medskip\noindent
{\bf Acknowledgements:} 
We are grateful to the anonymous referees for their useful comments and to Siddharth Srivastava, Zhun Yang, and Yu Zhang for helpful discussions. This work was partially supported by the National Science Foundation under Grant IIS-1815337.

\bibliographystyle{acmtrans}
\bibliography{bib,bib2}
\BOC

\EOC

\newpage

\appendix

\section{Proofs}

Proofs of Theorem \ref{thm:path-probability}, \ref{thm:transition-probability} and Corollary \ref{thm:reduce2transition} can be found in the supplimentary material of \cite{lee18aprobabilistic}.

\subsection{Propositions and Lemmas}

We write 
$\langle a_{0}, a_{1}\dots,  a_{m-1} \rangle^t$  (each $a_{i} \in {\bf A}$) to denote the formula 
$0\!:\!a_{0}\wedge 1\!:\!a_{1} \dots\wedge  m-1\!:\!a_{m-1}$. 
The following lemma tells us that any action sequence has the same probability under $Tr(D, m)$.

For any multi-valued probabilistic program $\Pi$,
let $pf_1, \dots, pf_m$ be the probabilistic constants in $\Pi$, and $v_{i,1}, \dots, v_{i, k_i}$, each associated with probability $p_{i,1}, \dots, p_{i, k_i}$ resp. be the values of $pf_i$ ($i\in\{1, \dots, m\}$). We use $TC_{\Pi}$ be the set of all assignments to probabilistic constants in $\Pi$. 
\begin{lemma}\label{lem:action-equal-probability}
For any $p\cal{BC}+$ action description $D$ and any action sequence $\vec{a} = \langle a_0, a_1, \dots, a_{m-1}\rangle$, we have
\[
P_{Tr(D, m)}(\vec{a}^t) = \frac{1}{(|\sigma^{act}|+1)^{m}}.
\]
\end{lemma}

\begin{proof}
\begin{align}
\nonumber & P_{Tr(D, m)}(\vec{a}^t)\\
\nonumber =\ & \underset{\substack{I\vDash \vec{a}^t\\ \text{$I$ is a stable models of $Tr(D, m)$}}}{\sum} P_{Tr(D, m)}(I)\\
\nonumber =\ & \text{(In $Tr(D, m)$ every total choice leads to $(|\sigma^{act}|+1)^{m}$ stable models. By Proposition 2 in \cite{lee18aprobabilistic}, )}\\
\nonumber \ & \underset{\substack{I\vDash \vec{a}^t\\ \text{$I$ is a stable models of $Tr(D, m)$}}}{\sum} \frac{W_{Tr(D, m)}(I)}{(|\sigma^{act}|+1)^{m}}\\
\nonumber =\ &\frac{\underset{tc\in TC_{Tr(D, m)}}{\sum} \underset{c=v\in tc}{\prod}M_{\Pi}(c=v)}{(|\sigma^{act}|+1)^{m}}\\
\nonumber =\ &\text{(Derivations same as in the proof of Proposition 2 in \cite{lee18aprobabilistic})}\\
\nonumber \ & \frac{1}{(|\sigma^{act}|+1)^{m}}
\end{align}
\end{proof}

The following lemma states that given any action sequence, the probabilities of all possible state sequences sum up to $1$.
\begin{lemma}\label{lem:state-sequence-prob-sum-up-to-1}
For any $p\cal{BC}+$ action description $D$ and any action sequence $\vec{a} = \langle a_0, a_1, \dots, a_{m-1}\rangle$, we have
\[
\underset{s_0, \dots, s_m: s_i\in {\bf S}}{\sum}P_{Tr(D, m)}(\langle s_0, \dots, s_m \rangle^t\mid \vec{a}^t) = 1.
\]
\end{lemma}

\begin{proof}
\begin{align}
\nonumber & \underset{s_0, \dots, s_m: s_i\in {\bf S}}{\sum}P_{Tr(D, m)}(\langle s_0, \dots, s_m \rangle^t\mid \vec{a}^t)\\
\nonumber =\ & \text{(By Corollary 1 in \cite{lee18aprobabilistic})}\\
\nonumber \ & \underset{s_0, \dots, s_m: s_i\in {\bf S}}{\sum}\underset{i\in\{0, \dots, m-1\}}{\prod} p(s_i, a_i, s_{i+1})\\
\nonumber =\ & \underset{s_0\in {\bf S}}{\sum} (p(s_0)\cdot \underset{s_1, \dots, s_m: s_i\in {\bf S}}{\sum}\underset{i\in\{1, \dots, m-1\}}{\prod} p(s_i, a_i, s_{i+1}))\\
\nonumber =\ & \underset{s_0\in {\bf S}}{\sum} (p(s_0)\cdot \underset{s_1\in {\bf S}}{\sum}(p(s_0, a_0, s_1)\cdot \underset{s_2, \dots, s_m: s_i\in {\bf S}}{\sum}\underset{i\in\{2, \dots, m-1\}}{\prod} p(s_i, a_i, s_{i+1})))\\
\nonumber =\ & \underset{s_0\in {\bf S}}{\sum} (p(s_0)\cdot \underset{s_1\in {\bf S}}{\sum}(p(s_0, a_0, s_1)\cdot \dots\cdot\underset{s_m \in {\bf S}}{\sum}p(s_{m-1}, a_i, s_{m})\dots))\\
\nonumber =\ & 1.
\end{align}
\end{proof}

The following proposition tells us that the probability of any state sequence conditioned on the constraint representation of a policy $\pi$ coincide with the probability of the state sequence conditioned on the action sequence specified by $\pi$ w.r.t. the state sequence.
\begin{prop}\label{prop:policy-to-action}
For any $p\cal{BC}+$ action description $D$, state sequence $\vec{s} = \langle s_0, s_1, \dots, s_m\rangle$, and a non-stationary policy $\pi$, we have
\begin{align}
\nonumber &P_{Tr(D, m)}(\vec{s}^t \mid C_{\pi, m}) =\\
\nonumber &P_{Tr(D, m)}(\vec{s}^t \mid \langle \pi(s_0, 0), \pi(s_1, 1), \dots,\pi(s_{m-1}, m-1)\rangle^t)
\end{align}
\end{prop}
\begin{proof}
\begin{align}
\nonumber &P_{Tr(D, m)}(\vec{s}^t \mid C_{\pi, m}) \\
\nonumber =\ & \frac{P_{Tr(D, m)}(\langle s_0,\dots, s_m\rangle^t \wedge C_{\pi, m})}{P_{Tr(D, m)}(C_{\pi, m})}\\
\nonumber =\ & \frac{P_{Tr(D, m)}(\langle s_0,\pi(s_0, 0)\dots, \pi(s_{m-1}, m-1), s_m\rangle^t}{P_{Tr(D, m)}(C_{\pi, m})}\\
\nonumber =\ & \frac{P_{Tr(D, m)}(\langle \pi(s_0, 0)\dots, \pi(s_{m-1}, m-1), s_m\rangle^t
\mid 0\!:\!s_0)\cdot P_{Tr(D, m)}(0\!:\!s_0)}{\underset{s'_0, \dots, s'_m: s'_i\in {\bf S}}{\sum}P_{Tr(D, m)}(\langle s'_0, \pi(s'_0, 0)\dots, \pi(s'_{m-1}, m-1), s'_m\rangle^t)}.
\end{align}

We use $k(s_0, \dots, s_m)$ as an abbreviation of 
\[
P_{Tr(D, m)}(\langle\pi(s_0, 0),\dots, \pi(s_{m-1}, m-1)\rangle^t).
\]
We have
\begin{align}
\nonumber & P_{Tr(D, m)}(\vec{s}^t \mid C_{\pi, m}) \\
\nonumber =\ & \frac{P_{Tr(D, m)}(\langle s_1, \dots , s_m\rangle^t
\mid \langle s_0, \pi(s_0, 0), \dots, \pi(s_{m-1}, m-1)\rangle^t)\cdot P_{Tr(D, m)}(0\!:\!s_0)\cdot k(s_0, \dots, s_m)}{\underset{s'_0, \dots, s'_m: s'_i\in {\bf S}}{\sum}P_{Tr(D, m)}(\langle s'_1, \dots , s'_m\rangle^t\mid \langle s'_0, \pi(s'_0, 0),\dots, \pi(s'_{m-1}, m-1)\rangle^t)\cdot P_{Tr(D, m)}(0\!:\!s'_0)\cdot k(s'_0, \dots, s'_m)}\\
\nonumber =\ & \text{(By Lemma \ref{lem:action-equal-probability}, for any $s_0, \dots, s_m (s_i\in {\bf S})$, we have $k(s_0, \dots, s_m)=\frac{1}{(\sigma^{act}|+1)^{m}}$)}\\
\nonumber \ & \frac{P_{Tr(D, m)}(\langle s_1, \dots , s_m\rangle^t
\mid \langle s_0, \pi(s_0, 0), \dots, \pi(s_{m-1}, m-1)\rangle^t)\cdot P_{Tr(D, m)}(0:s_0)\cdot \frac{1}{(|\sigma^{act}|+1)^{m}}}{\underset{s'_0, \dots, s'_m: s'_i\in {\bf S}}{\sum}P_{Tr(D, m)}(\langle s'_1, \dots , s'_m\rangle^t\mid \langle s'_0, \pi(s'_0, 0),\dots, \pi(s'_{m-1}, m-1)\rangle^t)\cdot P_{Tr(D, m)}(0\!:\!s'_0)\cdot \frac{1}{(|\sigma^{act}|+1)^{m}}}\\
\nonumber =\ &\frac{P_{Tr(D, m)}(\langle s_1, \dots , s_m\rangle^t
\mid \langle s_0, \pi(s_0, 0), \dots, \pi(s_{m-1}, m-1)\rangle^t)\cdot P_{Tr(D, m)}(0:s_0)}{\underset{s'_0, \dots, s'_m: s'_i\in {\bf S}}{\sum}P_{Tr(D, m)}(\langle s'_1, \dots , s'_m\rangle^t\mid \langle s'_0, \pi(s'_0, 0),\dots, \pi(s'_{m-1}, m-1)\rangle^t)\cdot P_{Tr(D, m)}(0\!:\!s'_0)}\\
\nonumber =\ & \text{(By Lemma \ref{lem:state-sequence-prob-sum-up-to-1}, the denominator equals $1$)}\\
\nonumber \ &P_{Tr(D, m)}(\langle s_1, \dots , s_m\rangle^t
\mid \langle s_0, \pi(s_0, 0), \dots, \pi(s_{m-1}, m-1)\rangle^t)\cdot P_{Tr(D, m)}(0:s_0)\\
\nonumber =\ &P_{Tr(D, m)}(\langle s_0, s_1, \dots, s_m\rangle^t \mid \langle \pi(s_0, 0), \dots, \pi(s_m-1, m-1)\rangle^t)
\end{align}
\end{proof}

\subsection{Proofs of Proposition \ref{prop:history-determines-utility}, Proposition \ref{prop:expected-utility-policy}, Theorem \ref{thm:sequence-utility-equivalence} and Theorem \ref{thm:pBC-plus-to-mdp}}

The following proposition tells us that, for any states and actions sequence, any stable model of $Tr(D, m)$ that satisfies the sequence has the same utility. Consequently, the expected utility of the sequence can be computed by looking at any single stable model that satisfies the sequence.

\noindent{\bf Proposition~\ref{prop:history-determines-utility} \optional{prop:history-determines-utility}}\
\ 
{\sl
For any two stable models $X_1, X_2$ of $Tr(D, m)$ that satisfy a history \\ $\vec{h} = \langle s_0, a_0, s_1, a_1, \dots,  a_{m-1}, s_m\rangle$, we have
\begin{align}
\nonumber &U_{Tr(D, m)}(X_1)\ = \  U_{Tr(D, m)}(X_2)  
\ = \
E[U_{Tr(D, m)}(\vec{h}^t)].
\end{align}
}
\begin{proof}
Since both $X_1$ and $X_2$ both satisfy $\vec{h}^t$,   $X_1$ and $X_2$ agree on truth assignment on $\sigma^{act}_m\cup\sigma^{fl}_m$. Notice that atom of the form ${\tt utility}(v, {\bf t})$ in $Tr(D, m)$ occurs only of the form \eqref{eq:utility-law-lpmln}, and only atom in $\sigma^{act}_m\cup\sigma^{fl}_m$ occurs in the body of rules of the form \eqref{eq:utility-law-lpmln}.
\begin{itemize}
\item Suppose an atom ${\tt utility}(v, {\bf t})$ is in $X_1$. Then the body $B$ of at least one rule of the form \eqref{eq:utility-law-lpmln} with ${\tt utility}(v, {\bf t})$ in its head in $Tr(D, m)$ is satisfied by $X_1$. $B$ must be satisfied by $X_2$ as well, and thus ${\tt utility}(v, {\bf t})$ is in $X_2$ as well.
\item Suppose an atom ${\tt utility}(v, {\bf t})$, is not in $X_1$. Then, assume, to the contrary, that ${\tt utility}(v, {\bf t})$ is in $X_2$, then by the same reasoning process above in the first bullet, ${\tt utility}(v, {\bf t})$ should be in $X_1$ as well, which is a contradiction. So ${\tt utility}(v, {\bf t})$ is also not in $X_2$.
\end{itemize}
So $X_1$ and $X_2$ agree on truth assignment on all atoms of the form ${\tt utility}(v, {\bf t})$, and consequently we have $U_{Tr(D, m)}(X_1) =   U_{Tr(D, m)}(X_2)$, as well as
\begin{align}
\nonumber &E[U_{Tr(D, m)}(\vec{h}^t)]\\
\nonumber =\ & \underset{I\vDash \vec{h}^t}{\sum} P_{Tr(D, m)}(I\mid \vec{h}^t)\cdot U_{Tr(D, m)}(I)\\
\nonumber =\ & U_{Tr(D, m)}(X_1)\cdot \underset{I\vDash \vec{h}^t}{\sum} P_{Tr(D, m)}(I\mid \vec{h}^t)\\
\nonumber =\ &\text{(The second term equals $1$)}\\
\nonumber \ & U_{Tr(D, m)}(X_1).
\end{align}
\end{proof}

The following proposition tells us that the expected utility of an action and state sequence can be computed by summing up the expected utility from each transition.
\begin{prop}\label{prop:step-wise-utility}
For any $p\cal{BC}+$ action description $D$ and a history $\vec{h} = \langle s_0, a_0, s_1, \dots, a_{m-1}, s_m\rangle$, such that there exists at least one stable model of $Tr(D, m)$ that satisfies $\vec{h}$, we have
\[
E[U_{Tr(D, m)}(\vec{h}^t)] = \underset{i\in \{0, \dots, m-1\}}{\sum} u(s_i, a_i, s_{i+1}).
\]
\end{prop}

\begin{proof}
Let $X$ be any stable model of $Tr(D, m)$ that satisfies $\vec{h}^t$. By Proposition \ref{prop:history-determines-utility}, we have
\begin{align}
\nonumber &E[U_{Tr(D, m)}(\vec{h}^t)]\\
\nonumber =\ & U_{Tr(D, m)}(X)\\
\nonumber =\ & \underset{i\in \{0, \dots, m-1\}}{\sum}(\underset{\substack{utility(v, i, {\bf x})\leftarrow (i+1:F)\wedge(i:G)\in Tr(D, m) \\ \text{$X$ satisfies $(i+1:F)\wedge(i:G)$}}}{\sum} v)\\
\nonumber =\ & \underset{i\in \{0, \dots, m-1\}}{\sum}(\underset{\substack{utility(v, 0, {\bf x})\leftarrow (1:F)\wedge(0:G)\in Tr(D, m) \\ \text{$0\!:\!X^i$ satisfies $(1:F)\wedge(0:G)$}}}{\sum} v)\\
\nonumber =\ & \underset{i\in \{0, \dots, m-1\}}{\sum}U_{Tr(D, 1)}(0\!:\!X^i)\\
\nonumber =\ & \text{( By Proposition \ref{prop:history-determines-utility})}\\
\nonumber \ & \underset{i\in \{0, \dots, m-1\}}{\sum}E[U_{Tr(D, 1)}(0:s_i, 0:a_i, 1:s_{i+1})]\\
\nonumber =\ & \underset{i\in \{0, \dots, m-1\}}{\sum} u(s_i, a_i, s_{i+1}).
\end{align}
\end{proof}

\noindent{\bf Proposition~\ref{prop:expected-utility-policy} \optional{prop:expected-utility-policy}}\
\ 
{\sl
Given any initial state $s_0$ that is consistent with $D_{init}$, for any policy $\pi$, we have
\[
\ba l
\small 
E[U_{Tr(D, m)}(C_{\pi, m}\wedge \langle s_0\rangle^t)] =  \\
   \underset{\vec{s} = \langle s_1, \dots, s_m\rangle: s_i\in {\bf S}}{\sum}
      R_D(\vec{h}_\pi(\vec{s})^t)
      \times P_{Tr(D, m)}(\vec{s}^t\wedge C_{\pi, m}).
\ea 
\]
}
\begin{proof}
We have
\begin{align}
\nonumber &E[U_{Tr(D, m)}(C_{\pi, m}\wedge \langle s_0\rangle^t)]\\
\nonumber =\ & \underset{I\vDash 0:s_0\wedge C_{\pi, m}}{\sum} P_{Tr(D, m)}(I\mid 0\!:\!s_0\wedge C_{\pi, m})\cdot U_{Tr(D, m)}(I)\\
\nonumber =\ & \underset{\substack{I\vDash 0:s_0\wedge C_{\pi, m}\\\text{$I$ is a stable model of $Tr(D, m)$}}}{\sum} P_{Tr(D, m)}(I\mid 0\!:\!s_0\wedge C_{\pi, m})\cdot U_{Tr(D, m)}(I)\\
\nonumber =\ & \text{(We partition stable models $I$ according to their truth assignment on $\sigma^{fl}_m$)}\\
\nonumber \ & \underset{\vec{s}=\langle s_1, \dots, s_m\rangle: s_i\in {\bf S}\ \ }{\sum} \underset{\substack{I\vDash \vec{s}^t\wedge C_{\pi, m}\\ \text{$I$ is a stable model of $Tr(D, m)$}}}{\sum} P_{Tr(D, m)}(I\mid 0\!:\!s_0\wedge C_{\pi, m})\cdot U_{Tr(D, m)}(I) \\
\nonumber =\ & \text{(Since $I\vDash \vec{s}^t\wedge C_{\pi, m}$ implies $I\vDash \vec{h}_\pi(\vec{s})^t$, by Proposition \ref{prop:history-determines-utility} we have)}\\
\nonumber \ & \underset{\vec{s} = \langle s_1, \dots, s_m\rangle: s_i\in {\bf S}\ \ }{\sum} \underset{\substack{I\vDash \vec{s}^t\wedge C_{\pi, m}\\ \text{$I$ is a stable model of $Tr(D, m)$}}}{\sum} P_{Tr(D, m)}(I\mid 0\!:\!s_0\wedge C_{\pi, m})\cdot E[U_{Tr(D, m)}(\vec{h}_\pi(s)^t)] \\
\nonumber =\ & \underset{\vec{s} = \langle s_1, \dots, s_m\rangle: s_i\in {\bf S}}{\sum}Pr_{Tr(D, m)}(\vec{s}^t \mid 0\!:\!s_0\wedge C_{\pi, m})\cdot E[U_{Tr(D, m)}(\vec{h}_\pi(s)^t)]\\
\nonumber =\ & \underset{\vec{s} = \langle s_1, \dots, s_m\rangle: s_i\in {\bf S}}{\sum}Pr_{Tr(D, m)}(\vec{s}^t \mid 0\!:\!s_0\wedge C_{\pi, m})\cdot E[U_{Tr(D, m)}(\vec{s}^t\wedge C_{\pi, m})]\\
\nonumber =\ & \underset{\vec{s} = \langle s_1, \dots, s_m\rangle: s_i\in {\bf S}}{\sum}
      R_D(\vec{h}_\pi(\vec{s})^t)
      \times P_{Tr(D, m)}(\vec{s}^t \wedge C_{\pi, m}).
\end{align}
\end{proof}

\noindent{\bf Theorem~\ref{thm:sequence-utility-equivalence} \optional{prop:sequence-utility-equivalence}}\
\ 
{\sl
Given an initial state $s_0\in {\bf S}$ that is consistent with $D_{init}$, for any non-stationary policy $\pi$ and any finite state sequence $\vec{s} = \langle s_0, s_1, \dots, s_{m-1}, s_m\rangle $ such that each $s_i$ in ${\bf S}\ (i\in \{0, \dots, m\})$, we have
\begin{itemize}
    \item $R_D(\vec{h}_\pi(\vec{s}))
           = R_{M(D)}(\vec{h}_\pi(\vec{s}))$
    \item $P_{Tr(D, m)}(\vec{s}^t \mid \langle s_0\rangle^t\wedge C_{\pi, m})
           = P_{M(D)}(\vec{h}_\pi(\vec{s}))$.
\end{itemize}}
\begin{proof}
We have
\begin{align}
\nonumber & R_D(\vec{h}_\pi(\vec{s}))\\
\nonumber =\ & E[U_{Tr(D, m)}(\vec{s}^t\wedge C_{\pi, m})]\\
\nonumber =\ & \text{(By Proposition \ref{prop:step-wise-utility})}\\
\nonumber  & \underset{i\in \{0, \dots, m-1\}}{\sum} u(s_i, \pi(s_i, i), s_{i+1})\\
\nonumber =\ & \underset{i\in \{0, \dots, m-1\}}{\sum} R(s_i, \pi(s_i, i), s_{i+1})\\
\nonumber =\ & R_{M(D)}(\vec{h}_\pi(\vec{s}))
\end{align}
and 
\begin{align}
\nonumber & P_{Tr(D, m)}(\vec{s}^t \mid \langle s_0\rangle^t\wedge C_{\pi, m})\\
\nonumber =\ & \text{(By Proposition \ref{prop:policy-to-action})}\\
\nonumber \ &Pr_{Tr(D, m)}(\vec{s}^t \mid \vec{h}_\pi(\vec{s})^t)\\
\nonumber =\ & \text{(By Corollary 1 in \cite{lee18aprobabilistic})}\\
\nonumber \ &\underset{i\in \{0, \dots, m-1\}}{\prod}p(\langle s_i, \pi(s_i, i), s_{i+1}\rangle)\\
\nonumber =\ &P_{M(D)}(\vec{h}_\pi(\vec{s}))
\end{align}
\end{proof}

\noindent{\bf Theorem~\ref{thm:pBC-plus-to-mdp} \optional{thm:pBC-plus-to-mdp}}\
\ 
{\sl
For any nonnegative integer $m$ and an initial state $s_0\in {\bf S}$ that is consistent with $D_{init}$, we have
\[
\underset{\text{$\pi$ is a policy}}{\rm argmax}\ E[U_{Tr(D, m)}(C_{\pi,  m}\wedge \langle s_0\rangle^t)]
=\underset{\pi}{\rm argmax}\ \i{ER}_{M(D)}(\pi, s_0).
\]
}

\begin{proof}
We show that for any non-stationary policy $\pi$, 
\[
E[U_{Tr(D, m)}(C_{\pi,  m}\wedge \langle s_0\rangle^t)] = \i{ER}_{M(D)}(\pi, s_0).
\]
We have
\begin{align}
\nonumber & E[U_{Tr(D, m)}(C_{\pi,  m}\wedge \langle s_0\rangle^t)] \\
\nonumber =\ & \text{(By Proposition \ref{prop:expected-utility-policy})}\\
\nonumber \ &  \underset{\vec{s}=\langle s_1, \dots, s_m\rangle: s_i\in {\bf S}}{\sum}
      R_D(\vec{h}_\pi(\vec{s}))
      \times P_{Tr(D, m)}(\vec{s}^t \mid \langle s_0\rangle^t\wedge C_{\pi, m}).\\
\nonumber =\ & \text{(By Theorem \ref{thm:sequence-utility-equivalence})}\\
\nonumber \ & \underset{\vec{s}=\langle s_1, \dots, s_m\rangle: s_i\in {\bf S}}{\sum}R_{M(D)}(\vec{h}_\pi(\vec{s}))\cdot P_{M(D)}(\vec{h}_\pi(\vec{s}))\\
\nonumber =\ & \i{ER}_{M(D)}(\pi, s_0).
\end{align}
\end{proof}

\section{{\sc pbcplus2mdp} System Description}\label{sec:pbcplus2mdp-appendix}

We describe the exact procedure performed by {\sc pbcplus2mdp} in Algorithm \ref{alg:pbcplus2mdp}. {\sc pbcplus2mdp} uses {\sc lpmln2asp}, which is component of {$\lpmln$ 1.0} system \cite{lee17computing}, for exact inference to find states, actions, transition probabilities and transition rewards. {\sc pbcplus2mdp} uses {\sc mdptoolbox} for solving the MDP generated from the input action description. 

\begin{algorithm}[h!]
{\footnotesize
\noindent {\bf Input: }
\begin{enumerate}
\item $Tr(D, m)$: A $\pbcp$ action description translated into $\lpmln$ program, parameterized with maxstep $m$, with states set ${\bf S}$ and action sets ${\bf A}$
\item $T$: time horizon
\item $\gamma$: discount factor
\end{enumerate}
\noindent {\bf Output: } Optimal policy

\noindent {\bf Procedure:}
\begin{enumerate}
\item Execute {\sc lpmln2asp} on $Tr(D, m)$ with $m=0$ to obtain all stable models of $Tr(D, 0)$; project each stable model of $Tr(D, 0)$ to only atoms corresponding to fluent constant (marked by {\tt fl\_} prefix); assign a unique number $idx(s)\in\{0, \dots, |{\bf S}|-1\}$ to each of the projected stable model $s$ of $Tr(D, 0)$;
\item Execute {\sc lpmln2asp} on $Tr(D, m)$ with $m=1$ and the clingo option {\tt --project} to project stable models to only atoms corresponding to action constant (marked by {\tt act\_} prefix); assign a unique number $idx(a)\in\{0, \dots, |{\bf A}|-1\}$ to each of the projected stable model $a$ of $Tr(D, 1)$;
\item Initialize 3-dimensional matrix $P$ of shape $(|{\bf A}|, |{\bf S}|, |{\bf S}|)$;
\item Initialize 3-dimensional matrix $R$ of shape $(|{\bf A}|, |{\bf S}|, |{\bf S}|)$;
\item For each state $s\in {\bf S}$ and action $a\in {\bf A}$:
\begin{enumerate}
\item execute {\sc lpmln2asp} on $Tr(D, m)\cup \{0:s\}\cup \{0:a\}\cup ST\_DEF$ with $m=1$ and the option {\tt -q "end\_state"}, where $ST\_DEF$ contains the rule
$$\{{\tt end\_state}(idx(s))\leftarrow 1:s\mid s\in {\bf S}\}.$$
\item Obtain $P_{Tr(D, 1)}(1: s' \mid 0:s, 0:a)$ by extracting the probability of $P_{Tr(D, 1)}({\tt end\_state}(idx(s')) \mid 0:s, 0:a)$ from the output;
\item $P(idx(a), idx(s), idx(s'))\leftarrow P_{Tr(D, 1)}(1: s' \mid 0:s, 0:a)$;
\item Obtain $E[U_{Tr(D, 1)}(1: s', 0:s, 0:a)]$ from the output by selecting an arbitrary answer set returned that satisfies $1: s'\wedge 0:s\wedge 0:a$ and sum up the first arguments of all atoms with predicate name {\tt utility} (By Proposition \ref{prop:history-determines-utility}, this is equivalent to $E[U_{Tr(D, 1)}(1: s', 0:s, 0:a)]$). 
\item $R(idx(a), idx(s), idx(s'))\leftarrow E[U_{Tr(D, 1)}(1: s', 0:s, 0:a)]$;
\end{enumerate}
\item Call finite horizon policy optimization algorithm of {\sc pymdptoolbox} with transition matrix $P$, reward matrix $R$, time horizon $T$ and discount factor $\gamma$; return the output.

\end{enumerate}
\caption{{\sc pbcplus2mdp} system}
\label{alg:pbcplus2mdp}
}
\end{algorithm}

The input is the $\lpmln$ translation $Tr(D, m)$ of a $\pbcp$ action description $D$, a time horizon $T$, and a discount factor $\gamma$. In the input $\lpmln$ program, we use atoms of the form ${\tt fl\_}x(v_1, \dots, v_m, {\tt t}, i)$, (and ${\tt fl\_}x(v_1, \dots, v_m, {\tt f}, i)$)  to encode fluent constant $x(v_1, \dots, v_m)$ is true, (and false, resp.) at time step $i$. Similarly, action constants and pf constants are encoded with atoms with prefix {\tt act\_} and {\tt pf\_}, resp. $Tr(D, m)$ is parametrized with maximum step $m$, for executing with different settings of maximum time step. As an example, The $\lpmln$ translation of the $\pbcp$ action description in Section \ref{sec:block-world} (robot and blocks) is listed in \ref{sec:block-encoding}. 

To construct the MDP instance $M(D) = \langle S, A, T, R\rangle$ corresponding to $D$, {\sc pbcplus2mdp} constructs the set $S$ of states, the set $A$ of actions, transition probability function $T$ and reward function $R$ one by one. 

By definition, states of $D$ are interpretations $I^{fl}$ of $\sigma^{fl}$ such that $0\!:\!I^{fl}$ are residual stable models of $D_0$. Thus, {\sc pbcplus2mdp} finds the states of $D$ by projecting the stable models of $Tr(D, 0)$ to atoms with prefix {\tt fl\_}. {\sc lpmln2asp} is executed to find the stable models of $Tr(D, 0)$. The {\sc clingo} option {\tt --option} is used to project stable models to only atoms with {\tt fl\_} prefix. Similarly, {\sc pbcplus2mdp} finds the actions of $D$ by projecting the stable models of $Tr(D, 1)$ to atoms with prefix {\tt act\_}.

The transition probability function $T$ and the reward function $R$ are represented by three dimensional matrices, specifying the transition probability and transition reward for each transition $\langle s, a, s'\rangle$. Transition probabilities are obtained by computing conditional probabilities $P_{Tr(D, 1)}(1\!:\!s'\mid 0\!:\!s, 0\!:\!a)$ for every transition $\langle s, a, s'\rangle$, using {\sc lpmln2asp}. Transition reward of each transition $\langle s, a, s'\rangle$ are obtained by computing the utility of any stable model of $Tr(D, 1)$ that satisfies $0:s\wedge 0:a\wedge 1:s'$. This is justified by Proposition \ref{prop:history-determines-utility}.

Finally, the constructed MDP instance $M(D)$, along with time horizon $T$ and discount factor $\gamma$, is used as input to {\sc mdptoolbox} to find the optimal policy.

The system has the following dependencies:
\begin{itemize}
\item {\sc Python} 2.7
\item {\sc clingo} python library: \url{https://github.com/potassco/clingo/blob/master/INSTALL.md}
\item {\sc lpmln2asp} system: \url{http://reasoning.eas.asu.edu/lpmln/index.html}
\item {\sc MDPToolBox}: \url{https://pymdptoolbox.readthedocs.io/en/latest/}
\end{itemize}

The system {\sc pbcplus2mdp}, source code, example instances and outputs can all be found at \url{https://github.com/ywang485/pbcplus2mdp}.

\BOCC
\begin{example}

We store this $\lpmln$ program as {\tt block.lpmln}. Executing {\sc pbcplus2mdp} system with the command line
\begin{lstlisting}
python pbcplus2mdp.py block.lpmln 3 0.9
\end{lstlisting}
finds the optimal policy for this instance with $3$ blocks with a time horizon of $3$ and a discount factor of $0.9$. The output specifies the best action to take for each state and each time step. We skip the complete output due to its length\footnote{The complete outputs for all our experiments can be found in \url{https://github.com/ywang485/pbcplus2mdp}}. As an example, the following is a small fragment of the output: 
\begin{lstlisting}
-------------------------- Time step 0 ---------------------------------:
state:  fl_TopClear(b1,t,0),fl_TopClear(b2,t,0),fl_TopClear(b3,t,0),fl_At(b1,r1,0),fl_At(b2,r1,0),fl_At(b3,r1,0),fl_Above(b1,b1,f,0),fl_Above(b2,b1,f,0),fl_Above(b3,b1,f,0),fl_Above(b1,b2,f,0),fl_Above(b2,b2,f,0),fl_Above(b3,b2,f,0),fl_Above(b1,b3,f,0),fl_Above(b2,b3,f,0),fl_Above(b3,b3,f,0),fl_GoalNotAchieved(t,0),fl_OnTopOf(b1,b1,f,0),fl_OnTopOf(b2,b1,f,0),fl_OnTopOf(b3,b1,f,0),fl_OnTopOf(b1,b2,f,0),fl_OnTopOf(b2,b2,f,0),fl_OnTopOf(b3,b2,f,0),fl_OnTopOf(b1,b3,f,0),fl_OnTopOf(b2,b3,f,0),fl_OnTopOf(b3,b3,f,0)
action:  act_StackOn(b1,b3,t,0),act_StackOn(b1,b1,f,0),act_StackOn(b2,b1,f,0),act_StackOn(b3,b1,f,0),act_StackOn(b1,b2,f,0),act_StackOn(b2,b2,f,0),act_StackOn(b3,b2,f,0),act_StackOn(b2,b3,f,0),act_StackOn(b3,b3,f,0),act_MoveTo(b1,r1,f,0),act_MoveTo(b2,r1,f,0),act_MoveTo(b3,r1,f,0),act_MoveTo(b1,r2,f,0),act_MoveTo(b2,r2,f,0),act_MoveTo(b3,r2,f,0)
...

\end{lstlisting}
The above fragment indicates, at time step 0, the best action to take if all three blocks are at {\tt R1}, and no block is on top of another block, is to stack block {\tt B1} on {\tt B3}.
\end{example}
\EOCC

\section{{\sc pbcplus2mdp} Input Encoding of the Robot and Block Example}\label{sec:block-encoding}

\begin{lstlisting}
astep(0..m-1).
step(0..m).
boolean(t; f).

block(b1; b2; b3).
location(l1; l2).

%% UEC
:- fl_Above(X1, X2, t, I), fl_Above(X1, X2, f, I).
:- not fl_Above(X1, X2, t, I), not fl_Above(X1, X2, f, I), block(X1), block(X2), step(I).
:- fl_TopClear(X, t, I), fl_TopClear(X, f, I).
:- not fl_TopClear(X, t, I), not fl_TopClear(X, f, I), block(X), step(I).
:- fl_GoalNotAchieved(t, I), fl_GoalNotAchieved(f, I).
:- not fl_GoalNotAchieved(t, I), not fl_GoalNotAchieved(f, I), step(I).

:- fl_At(X, L1, I), fl_At(X, L2, I), L1 != L2.
:- not fl_At(X, l1, I), not fl_At(X, l2, I), block(X), step(I).
:- fl_OnTopOf(X1, X2, t, I), fl_OnTopOf(X1, X2, f, I).
:- not fl_OnTopOf(X1, X2, t, I), not fl_OnTopOf(X1, X2, f, I), block(X1), block(X2), step(I).

:- act_StackOn(X1, X2, t, I), act_StackOn(X1, X2, f, I).
:- not act_StackOn(X1, X2, t, I), not act_StackOn(X1, X2, f, I), block(X1), block(X2), astep(I).
:- act_MoveTo(X, L, t, I), act_MoveTo(X, L, f, I).
:- not act_MoveTo(X, L, t, I), not act_MoveTo(X, L, f, I), block(X), location(L),astep(I).

:- pf_Move(t, I), pf_Move(f, I).
:- not pf_Move(t, I), not pf_Move(f, I), astep(I).

% ---------- PF(D) ----------
%% Probability Distribution
@log(0.8) pf_Move(t, I) :- astep(I).
@log(0.2) pf_Move(f, I) :- astep(I).

%% Initial State and Actions are Random
{fl_OnTopOf(X1, X2, B, 0)} :- block(X1), block(X2), boolean(B).
{fl_At(X, L, 0)} :- block(X), location(L), boolean(B).
{act_StackOn(X1, X2, B, I)} :- block(X1), block(X2), boolean(B), astep(I).
{act_MoveTo(X, L, B, I)} :- block(X), location(L), boolean(B), astep(I).

%% No Concurrency
:- act_StackOn(X1, X2, t, I), act_StackOn(X3, X4, t, I), astep(I), X1 != X3.
:- act_StackOn(X1, X2, t, I), act_StackOn(X3, X4, t, I), astep(I), X2 != X4.
:- act_MoveTo(X1, L1, t, I), act_MoveTo(X2, L2, t, I), astep(I), X1 != X2.
:- act_MoveTo(X1, L1, t, I), act_MoveTo(X2, L2, t, I), astep(I), L1 != L2.
:- act_StackOn(X1, X2, t, I), act_MoveTo(X3, L, t, I), astep(I).

%% Static Laws
fl_GoalNotAchieved(t, I) :- fl_At(X, L, I), L != l2.
fl_GoalNotAchieved(f, I) :- not fl_GoalNotAchieved(t, I), step(I).
:- fl_OnTopOf(X1, X, t, I), fl_OnTopOf(X2, X, t, I), X1 != X2.
:- fl_OnTopOf(X, X1, t, I), fl_OnTopOf(X, X2, t, I), X1 != X2.
fl_Above(X1, X2, t, I) :- fl_OnTopOf(X1, X2, t, I).
fl_Above(X1, X2, t, I) :- fl_Above(X1, X, t, I), fl_Above(X, X2, t, I).
:- fl_Above(X1, X2, t, I), fl_Above(X2, X1, t, I).
fl_At(X1, L, I) :- fl_Above(X1, X2, t, I), fl_At(X2, L, I).
fl_Above(X1, X2, f, I) :- not fl_Above(X1, X2, t, I), block(X1), block(X2), step(I).
fl_TopClear(X, f, I) :- fl_OnTopOf(X1, X, t, I).
fl_TopClear(X, t, I) :- not fl_TopClear(X, f, I), block(X), step(I).

%% Fluent Dynamic Laws
fl_At(X, L, I+1) :- act_MoveTo(X, L, t, I), pf_Move(t, I), fl_GoalNotAchieved(t, I).
fl_OnTopOf(X1, X2, t, I+1) :- act_StackOn(X1, X2, t, I), X1 != X2, fl_TopClear(X2, t, I), not fl_Above(X2, X1, t, I), fl_At(X1, L, I), fl_At(X2, L, I), fl_GoalNotAchieved(t, I).
fl_OnTopOf(X1, X2, f, I+1) :- act_MoveTo(X1, L2, t, I), pf_Move(t, I), fl_At(X1, L1, I), fl_OnTopOf(X1, X2, t, I), L1 != L2, fl_GoalNotAchieved(t, I).
fl_OnTopOf(X1, X, f, I+1) :- act_StackOn(X1, X2, t, I), X1 != X2, fl_TopClear(X2, t, I), not fl_Above(X2, X1, t, I), fl_At(X1, L, I), fl_At(X2, L, I), fl_OnTopOf(X1, X, t, I), X != X2, fl_GoalNotAchieved(t, I).
{fl_OnTopOf(X1, X2, B, I+1)} :- fl_OnTopOf(X1, X2, B, I), astep(I), boolean(B).
{fl_At(X, L, I+1)} :- fl_At(X, L, I), astep(I), boolean(B).

%% Utility Laws
utility(-1, X, L, I) :- act_MoveTo(X, L, t, I).
utility(10) :- fl_GoalNotAchieved(f, I+1), fl_GoalNotAchieved(t, I).
\end{lstlisting}

\end{document}